\documentclass{article}
\usepackage[utf8]{inputenc} 
\usepackage[T1]{fontenc}    
\usepackage{hyperref}       
\usepackage{url}            
\usepackage{booktabs}       
\usepackage{amsfonts}       
\usepackage{nicefrac}       
\usepackage{microtype}      
\usepackage{xcolor}         
\colorlet{darkgreen}{green!45!black}
\hypersetup{
  colorlinks,
  citecolor=darkgreen}
\usepackage{graphicx}
\usepackage{subfigure}
\usepackage{enumitem}
\usepackage{natbib}

\usepackage{algorithmic}
\usepackage{algorithm}
\renewcommand{\algorithmiccomment}[1]{\textcolor{green!50!black}{\hfill$\triangleright $#1}}

\usepackage{amsmath}
\usepackage{amssymb}
\usepackage{mathtools}
\usepackage{amsthm}
\usepackage{thmtools, thm-restate}

\usepackage[capitalize,noabbrev]{cleveref}

\theoremstyle{plain}
\declaretheorem{theorem}
\declaretheorem[sibling=theorem, numberwithin=section]{lemma}
\declaretheorem[sibling=theorem, numberwithin=section]{corollary}
\theoremstyle{definition}
\declaretheorem[sibling=theorem, numberwithin=section]{definition}
\declaretheorem[sibling=theorem, numberwithin=section]{assumption}
\declaretheorem[sibling=theorem, numberwithin=section]{proposition}
\newcommand{\defeq}{\overset{\text{\tiny def}}{=}}
\newcommand\indep{\protect\mathpalette{\protect\indepT}{\perp}}
\def\indepT#1#2{\mathrel{\rlap{$#1#2$}\mkern2mu{#1#2}}}

\usepackage[normalem]{ulem}

\usepackage{amsfonts}
\usepackage{xcolor}
\usepackage{bbold}
\usepackage{bigints}
\usepackage{stmaryrd}
\usepackage{booktabs,tabularx}
\usepackage{makecell}
\usepackage[nobiblatex]{xurl}
\newcommand\rurl[1]{%
  \href{http://#1}{\nolinkurl{#1}}%
}

\DeclareMathOperator*{\argmax}{arg\,max}
\DeclareMathOperator*{\argmin}{arg\,min}

\usepackage{tikz}
\usetikzlibrary{backgrounds}
\usetikzlibrary{tikzmark}
\usetikzlibrary{calc}
\usetikzlibrary{arrows,shapes,positioning,shadows,trees,mindmap}
\usetikzlibrary{arrows.meta}

\usepackage{tcolorbox}
\newcommand{\highlight}[2]{\colorbox{#1!17}{$\displaystyle #2$}}

\renewcommand{\highlight}[2]{\colorbox{#1!17}{#2}}

\usepackage[preprint]{neurips_2024}

\title{Teaching Models To Survive:\\Proper Scoring Rule and Stochastic Optimization with Competing Risks}

\author{%
  Julie Alberge \\
  SODA Team, Inria Saclay \\
  Palaiseau, France\\
  \texttt{julie.alberge@inria.fr} \\
  \And
  Vincent Maladière \\
  :probabl. \\
  Paris, France\\
  \texttt{vincent@probabl.ai} \\
  \And
  Olivier Grisel \\
  :probabl. \\
  Paris, France\\
  \AND
  Judith Abécassis \\
  SODA Team, Inria Saclay \\
  Palaiseau, France\\
  \And
  Gaël Varoquaux \\
  SODA Team, Inria Saclay \\
  Palaiseau, France\\
}

\begin{document}

\maketitle

\begin{abstract}
When data are right-censored, \emph{i.e.} some outcomes are missing due to a limited period of observation, survival analysis can compute the ``time to event''. Multiple classes of outcomes lead to a classification variant: predicting the most likely event, known as competing risks, which has been less studied.
To build a loss that estimates outcome probabilities for such settings, we introduce a strictly proper censoring-adjusted separable scoring rule that can be optimized on a subpart of the data because the evaluation is made independently of observations. It enables stochastic optimization for competing risks which we use to train gradient boosting trees. Compared to 11 state-of-the-art models, this model, MultiIncidence, performs best in estimating the probability of outcomes in survival and competing risks. It can predict at any time horizon and is much faster than existing alternatives.
\end{abstract}

\section{Introduction}

We all die at some point; some applications call for predicting not whether an event of interest will happen or not, but when it is likely to occur:  \emph{time-to-event regression}. In such a setting, samples often have unobserved outcomes, \emph{e.g.} individuals that have not been followed long enough for the event of interest to occur. Limiting the analysis to fully observed samples creates a censoring bias; valid models use dedicated corrections for censorship: \emph{survival analysis} models. These have long been central to health \citep{zhu_deep_2016, chaddad_radiomic_2016, gaynor_use_1993}. Nowadays, survival analysis is also used in diverse fields, such as predictive maintenance \citep{rith_analysis_2018, susto_machine_2015}, or user-engagement studies \citep{maystre_temporally-consistent_nodate}. 
Survival analysis has led to many dedicated models, such as the \citet{kaplan_nonparametric_1958} estimator or the \citet{cox_regression_1972} proportional hazard model. 

Competing risks analysis generalizes survival analysis to account for multiple events, determining which will happen first \citep{susto_machine_2015, gaynor_use_1993}. For instance, if a person with breast cancer dies from a different cause, it is impossible to determine when they would have succumbed to cancer, regardless of the duration of the observation period. \citep{seer}. The caregiver may also want to adapt the treatment if the patient is predicted to die of a competing event such as a heart attack sooner than from cancer.
As the risks of the various events are seldom independent--for instance, cancer and cardiovascular disease share inflammation or age risk factors \citep{koene2016shared}--competing risks cannot be solved by running a survival model for each event \citep{wolbers2009prognostic}. The estimated risk of a single event of interest will be biased if competing risks are not included. Hence, adequate models for these risks are critical for decision-making~\citep{ramspek2022lessons,koller2012competing,van2016competing}.

Survival models have traditionally been developed with ad hoc adjustments for censoring. The most common approach is to design a likelihood using the probability of censoring per unit time--\emph{i.e.} the time-derivative of the risk--which either comes with strong parametric assumptions \citep{cox_regression_1972} or ad hoc corrections \citep{wang_survtrace_2022}. Given that the risk, which is the probability of the outcome at a specific time, is crucial for various applications, it can be preferable to use losses that directly control probabilities (proper scoring rules), as developed by \citet{graf_assessment_1999,rindt_survival_2022}. However, no metric (or loss) has been shown to control probabilities in the competing risks setting.

In application domains typical of survival analysis and competing risks --health, predictive maintenance, insurance, marketing-- the data are typically tabular with categorical variables, where tree-based models shine \citep{grinsztajn_why_2022}.  Existing survival and competing risks models do not fit well with these requirements. In particular, the proper scoring rule introduced by \citet{rindt_survival_2022} requires a time derivative of the risk, typically via an auto-diff operator in a neural architecture. This approach is challenging to adapt to tree-based algorithms. In addition, the ever-growing volume of data calls for computationally efficient algorithms.

\paragraph{Contributions}
Here, we provide a general theoretical framework to learn a competing risks model with a proper scoring rule.  This scoring rule gives a loss easy to plug into any multiclass estimator to create a competing risks model: giving the individual risk of each event at any horizon. 
We also sum over time for model evaluation, as the resulting Integrated Scoring Rule is also proper.\\ 
 An interesting property of this new loss is that it can be optimized on a subset of the training data because the evaluation is made independently of observations. Hence, it allows stochastic optimization, enabling computationally efficient learning.
With that, we propose an algorithm called MultiIncidence, based on Stochastic Gradient Boosting Trees. 
We benchmark our algorithm on a synthetic dataset with varying censoring rates, number of features, and number of training samples to show that our method outperforms state-of-the-art methods while exhibiting faster training times. Finally, applying our model to real-life datasets demonstrates that it outperforms other models in both the competing risks context and basic survival analysis.

\section{Related work}
\paragraph{Survival settings}
Various survival models have been developed, ranging from approaches like the \citet{kaplan_nonparametric_1958} estimator, estimating the general survival curve of a whole population, to models that account for covariates. One of them is the \citet{cox_regression_1972} Proportional-Hazards Model, a linear model of \emph{hazard}: the instantaneous probability of an event, \emph{i.e.} the logarithmic derivative of outcome probabilities in time. More complex models have been adapted to the survival setting: Support Vector Machines \citep{van_belle_support_2011}, survival games \citep{han2021inverse} and neural networks with DeepSurv \citep{katzman_deepsurv_2018} or PCHazard \citep{kvamme_continuous_2019}. Although the above do not control risks,  more recent neural networks use adequate losses (see below): DQS \citep[though relying on a piecewise constant hazard]{yanagisawa2023proper}, SumoNet \citep[which requires differentiable models]{rindt_survival_2022}.

\paragraph{Competing risks}
Competing risks, with multiple outcomes, require new methods (which can naturally adapt to the simpler survival setting).
Derived from the \citet{kaplan_nonparametric_1958} estimator, the \citet{nelson_theory_1972}-\citet{aalen_survival_2008} estimator is an unbiased marginal model for competing risks. \\
The linear \citet{fine_proportional_1999} estimator is inspired by the \citet{cox_regression_1972} estimator in survival analysis and is the most used model in clinical research.
Machine-learning models have recently been adapted to the competing risks setting, including tree-based approaches such as the Random Survival Forests \citep{ishwaran_random_2008, kretowska_tree-based_2018, bellot_tree-based_2018}, boosting approaches \citep{bellot_multitask_2018}, neural networks approaches \textit{e.g.} DeepHit and Gaussian mixtures approaches \citep{lee_deephit_2018, aala_deep_2017, danks_derivative-based_2022, nagpal2021deep}  and transformers approaches with SurvTRACE \citep{wang_survtrace_2022} using a loss corrected to predict rare competing events but independently forecasts all events without ensuring probabilities sum to one.\\
For a review of the competing setting, the reader can refer to \citet{monterrubio-gomez_review_2022}.

\paragraph{Evaluation for such models}
Prediction evaluation in survival or competing risks settings calls for adapted metrics to account for right-censored points \citep{harrell}, like the C-index which adapts the Area Under the ROC curve in classification. However, the C-index only evaluates the ranking of samples, \emph{i.e.} which samples will undergo the event of interest first, and is dependent on the censoring distribution which may bias the evaluation~\citep{blanche_c-index_2019, rindt_survival_2022}. In fact, the score may be higher for distributions other than oracle-censoring distributions. Alternative methods have been proposed such as the \emph{time-dependent} C-index, $C_\zeta$ \citep{antolini_timedependent_2005}, which is the same metric but computed at a given time horizon $\zeta$. The C-index ranking metric has been extended to competing risks~\citep{uno_cstatistics_2011} but, as in the survival setting, the C-index only evaluates relative risks for pairs of individuals and not the absolute value of the risk for a given individual. 
Other time-dependent adaptations of the ROC curve have been developed, also assessing a discriminative power rather than risks or probabilities~\citep{blanche2013estimating}.
And yet control of the risk is crucial to decision making \citep{van2019calibration}. Proper scoring rules are alternatives to overcome the limitations of existing metrics because they capture more aspects of the problem. In addition, they can be used for both the training and evaluation of probabilistic predictive models.

\paragraph{Proper Scoring Rules (PSR)}
Scoring rules are functions of observations and a candidate probability distribution; when \emph{proper} they control for the oracle probability distribution (definition \ref{def:proper}).
This is important in machine learning to create losses that recover probabilities of outcomes.
For classification, where discrete events are observed rather than the probability, the Brier score and the log loss give proper scoring rules, with relative merits \citep{benedetti_scoring_2010,merkle_choosing_2013}. 

\citet{graf_assessment_1999} adapt the Brier score to survival analysis, with a strong independence assumption on the censoring distribution. However, the assumption can easily be violated \citep{kvamme_brier_2019} which leads to bias \citep{rindt_survival_2022}.
\citet{rindt_survival_2022} show that the likelihood of the survival function leads to a proper scoring rule but requires obtaining the density function and the survival function, a time-wise derivative of outcome probabilities (definition \ref{defs}).  For quantile regression, \citet{yanagisawa2023proper} shows that the Pinball loss may lead to a proper scoring rule for survival analysis but requires an oracle parameter.
\citet{han2021inverse} introduces a double optimization problem for which the stationary point is located at the true distributions.

For competing risks, \citet{schoop_quantifying_2011} extend the Brier score to a proper scoring rule. However, the Brier score does not measure the uncertainty as well as the log loss \citep{benedetti_scoring_2010}. 

\section{Problem Formulation}
\paragraph{Notations}
We write oracle quantities as $a^*$ and estimates as $\hat{a}$, vectors in bold, $\mathbf{a}$, random variables in upper case, $A$, observations in lower cases $a$, and distributions in calligraphy style $\mathcal{A}$.
\subsection{Problem setting}

We consider $K$ competing events and for $k \in \llbracket 1, K \rrbracket$, we denote $T^*_k \in \mathbb{R}_+$ the event time of the event $k$, depending on the covariate $\mathbf{X} \in \mathcal{X}$. We also denote $T^* \in \mathbb{R}_+, T^* = \min\limits_{k \in \llbracket 1, K \rrbracket}(T^*_k)$ and $\Delta^* \in \llbracket1, K \rrbracket, \Delta^* = \argmin\limits_{k \in \llbracket 1, K \rrbracket}(T^*_k)$.
We observe $(\mathbf{X}, T, \Delta) \sim \mathcal{D}$, with $T = \min(T^*, C)$ where $C \in \mathbb{R}_+$ is the censoring time, which may depend on $\mathbf{X}$ and $\Delta \in \llbracket0, K \rrbracket, \Delta= \argmin\limits_{k \in \llbracket 0, K \rrbracket}(T^*_k) $ where 0 denotes a censored observation. 
However, we are interested in the distribution of the uncensored data~$(\mathbf{X}, T^*, \Delta^*) \sim \mathcal{D}^*$ especially in the joint distribution of $T^*, \Delta^*|\mathbf{X} =\mathbf{x}$ and the marginal distribution of $T^*|\mathbf{X}= \mathbf{x}$. \\
This paper aims to predict an unbiased estimate of all of the cause-specific Cumulative Incidence functions (CIF) at any time horizon $\zeta$ chosen based on the observations $(\mathbf{x}, t, \delta)$:
\begin{definition}[\emph{Quantities of interest}]\label{defs}
    \begin{flalign*}%
    \text{\footnotesize{CIF} \small (cumulative incidence
    function):}&& F^*(\zeta|\mathbf{x}) = \mathbb{P}(T^* \leq \zeta
|\mathbf{X}= \mathbf{x}) &&\\[.1em]
    \text{\footnotesize{$k^{th}$ CIF}:}&& F^*_k(\zeta|\mathbf{x}) = \mathbb{P}(T^* \leq \zeta \cap \Delta^* = k |\mathbf{X}= \mathbf{x})  && \\[.1em]
    \text{\footnotesize{Censoring}:}&& G^*(\zeta|\mathbf{x}) = \mathbb{P}(C> \zeta |\mathbf{X}= \mathbf{x})  &&\\[.1em]
    \text{\footnotesize{Survival to any event}:}&& S^*(\zeta|\mathbf{x}) = \mathbb{P}(T^* > \zeta |\mathbf{X}= \mathbf{x}) && 
\end{flalign*}
\end{definition}%
\begin{assumption}[\emph{Non informative censoring}]\label{info_censoring}
     We make the classic assumption of survival analysis that the censoring is noninformative according to the covariates: $$\forall k, \in \llbracket1, K\rrbracket, ~~ T^*_k \indep C |\mathbf{X}$$
\end{assumption} 

Assumption~\ref{info_censoring} needed for most theoretical results in survival \citep{rindt_survival_2022, yanagisawa2023proper, han2021inverse}. It is key to understanding why single-event survival analysis is invalid in the presence of competing risks: if some observations are censored due to other events sharing unobserved risk factors with the event of interest, this assumption is violated.

\subsection{CIF scoring rule}
\paragraph{Proper Scoring Rule }
A scoring rule $\ell$ evaluates a distribution $\mathcal{P}$ on an observation $Y$ and gives a corresponding score $\ell(\mathcal{P}, Y)$. The better the score, the better the model fits the observation. For a proper scoring rule, it corresponds to the degree to which the model can predict the oracle distribution
\citep[more on scoring rules in][]{gneiting_strictly_2007, ovcharov_proper_2018, merkle_choosing_2013}. 
\begin{definition}[\emph{Proper Scoring Rule}]
A scoring rule $\ell$ is proper if
$$
\forall \mathcal{P}, \mathcal{Q}, \text{distributions} \qquad 
\mathbb{E}_{Y\sim \mathcal{Q}}[\ell(\mathcal{P}, Y)] \leq \mathbb{E}_{Y\sim \mathcal{Q}}[\ell(\mathcal{Q}, Y)]
$$
When equality is reached if and only if $\mathcal{P} = \mathcal{Q}$, the scoring rule is called strictly proper.

\label{def:proper}%
\end{definition} 

\paragraph{Proper scoring rule for the Global CIF}
We will denote $L_{\zeta}$, a scoring rule for the global CIF at a time horizon $\zeta$. 
\begin{definition}[\emph{PSR for competing risks settings}]\label{def:psrcr}
In competing events settings, as we face censoring, a scoring rule $L_{\zeta}$ for the CIF at time $\zeta$ for an observation $(\mathbf{X}, T, \Delta)$ is proper if and only if:
\begin{multline}
   \forall \zeta, (\mathbf{X}, T, \Delta )\sim \mathcal{D}, \\
    \mathbb{E}_{T, \Delta  | \mathbf{X}= \mathbf{x}}
    [L_{\zeta}(
    \tikzmarknode{estimands}{\highlight{pink}
        {$(\hat{F}_1(\zeta| \mathbf{x}), ..., \hat{F}_K(\zeta| \mathbf{x}), \hat{S}(\zeta| \mathbf{x})), (T, \Delta)$}
        }
    )]  
        \leq \quad \\ 
    \mathbb{E}_{T, \Delta  | \mathbf{X}= \mathbf{x}}[L_{\zeta}(\tikzmarknode{oracle}{\highlight{purple}{$(F^*_1(\zeta| \mathbf{x}), ..., F^*_K(\zeta| \mathbf{x}), S^*(\zeta| \mathbf{x}))$}, (T, \Delta)})] 
\end{multline}
\begin{tikzpicture}[overlay,remember picture,>=stealth,nodes={align=left,inner ysep=1pt},<-]
     \path (estimands.south) ++ (15em,3em) node[anchor=east,color=pink!200] (estititle){\text{Estimated distributions}};
     \draw [color=pink!200](estimands.north) -- ([xshift=-0.1ex,color=pink!200]estititle.west);
    \path (oracle.south) ++ (-3em,-1em) node[anchor=east,color=purple!67] (oracletitle){\text{Oracle distributions}};
     \draw [color=purple!87](oracle.south) -- ([xshift=-0.1ex,color=purple!87]oracletitle.east);
\end{tikzpicture}
\end{definition}

\section{A Proper Scoring Rule for Competing Risks}

We prove that the negative log-likelihood re-weighted by the censoring distribution (IPCW) is proper.

\begin{definition}[Competitive Weights Negative LogLoss]
    We introduce the multiclass negative log-likelihood re-weighted with the censoring distribution. The different classes represent the loss of all the cumulative incidence functions as well as the survival function.
    \begin{multline}
    \quad\forall \zeta, (\mathbf{x}, t, \delta) \sim \mathcal{D}, \quad \mathrm{L}_{\zeta}((\hat{F}_1(\zeta| \mathbf{x}), ..., \hat{F}_K(\zeta| \mathbf{x}), \hat{S}(\zeta| \mathbf{x})), (t, \delta)) \defeq \\
    \frac{1}{n} \sum_{i=1}^n \sum_{k=1}^{K} \left(
    \dfrac{
        \mathbb{1}_{t_i \leq \zeta, \delta_i = k} ~~\log\left(\hat{F}_k(\zeta|\mathbf{x}_i)\right)
        }
        {
        \tikzmarknode{censure_i}{\highlight{orange}
            {$G^*(t_i|\mathbf{x}_i) $}
            }} \right)
        +
        \dfrac{
        \mathbb{1}_{t_i > \zeta} ~~ \log\left(\hat{S}(\zeta|\mathbf{x}_i)\right)
        }
        {
        \tikzmarknode{censure_t}{\highlight{cyan}
            {$G^*(\zeta|\mathbf{x}_i) $}
        }}
    \label{eqn:full_loss}
\end{multline}
\begin{tikzpicture}[overlay,remember picture,>=stealth,nodes={align=left,inner ysep=1pt},<-]
     \path (censure_i.south) ++ (-5em, -1.5em) node[anchor=east,color=orange!67] (titlecensi){\text{Probability of remaining at $t_i$}};
     \draw [color=orange!87](censure_i.west) -- ([xshift=-0.1ex,color=orange]titlecensi.east);
    \path (censure_t.south) ++ (-3em,-1.5em) node[anchor=east,color=cyan!67] (censi){\text{\parbox{25ex}{Probability  of remaining at $\zeta$ \\ \small (1 - probability of  censoring)}}};
     \draw [color=cyan!87](censure_t.west) -- ([xshift=-0.1ex,color=cyan]censi.east);
\end{tikzpicture}
\end{definition}
\vspace{1.2em}

Eqn.\ref{eqn:full_loss} can be seen as a standard log-loss (a.k.a cross-entropy), reweighted by appropriate sample weights, the inverse probabilities, IPCW (inverse probabilities of censoring weights). It can thus be easily added to most multiclass estimators.
\smallskip

\begin{restatable}{lemma}{usefullemma}
\label{lem:usefulllemma}%
Accounting for the time horizon $\zeta$, the expectation of the above scoring rule can be written as: 
    $\quad\forall \zeta, (\mathbf{X}, T, \Delta) \sim \mathcal{D},$
    \begin{equation}
        \mathbb{E}_{T, \Delta |\mathbf{X} =\mathbf{x}}\left[\mathrm{L}_{\zeta}\left(\hat{F}_k(\zeta|\mathbf{x}), (T, \Delta)\right)\right] = 
    \sum_{k=1}^K \log\left(\hat{F}_k(\zeta|\mathbf{x})\right)
        F^*_k(\zeta |\mathbf{x})  
     + \log\left(\hat{S}(\zeta|\mathbf{x})\right)
        S^*(\zeta| \mathbf{x})
        \label{eqn:loss}
    \end{equation}
\end{restatable}%
\begin{proof}[Proof sketch]
    The weights enable moving from the observation distribution $T$ to the distribution of $T^*$, a key ingredient to show properness. The whole proof can be found in Appendix \ref{prooflemma}.
\end{proof}%
\begin{restatable}[Properness of the scoring rule]{theorem}{bigthm}\label{thm:bigtheorem}
    Under the assumption that the weights are well chosen, 
    $L_{\zeta}:  \mathbb{R}^{K+1} \times \mathcal{D} \rightarrow \mathbb{R}$ is a strictly proper scoring rule for the global CIF on a fixed time horizon $\zeta \in \mathbb{R}_+$.
\end{restatable}%
\begin{proof}[Proof sketch]
    With the previous result, the properties of the negative log-likelihood, and the Definition \ref{def:psrcr}, we obtain that the loss is strictly proper. The whole proof can be found in Appendix \ref{psrweights}.
\end{proof}

\section{MultiIncidence Model: Gradient boosting for competing risks}

\begin{figure*}[b]
\centerline{\includegraphics[width=\textwidth]{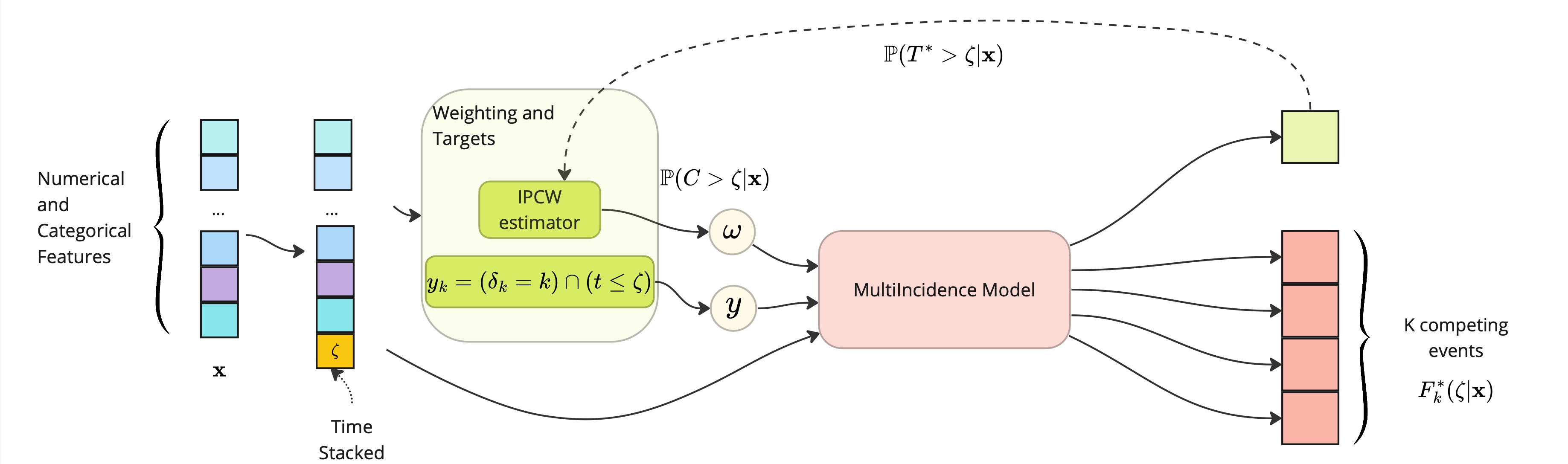}}

\caption{\textbf{MultiIncidence Model with its Feedback Loop.} After giving the input to the model, a random time is given and the weights and the target can be computed. After one iteration, the feedback loop trains the censoring probability -- $G^\star$ in eq.\ref{eqn:full_loss}.}
\label{architecture}
\end{figure*}

While eq.\ref{eqn:full_loss} can be used as a loss in any multiclass machine learning algorithm, we chose Gradient Boosting trees because of their performance on tabular data \citep{grinsztajn_why_2022} and their ability to be fit via stochastic optimization. Most survival or competing risk loss cannot be used with such tree-based models as the require time-derivates and thus smoothness.

We introduce a model, MultiIncidence, predicting all of the CIF for each competing event as well as the global survival function. Predicting these jointly easily maintains the stability of the probabilities as outputs of classifications model sum to one and $\mathbb{P}(T^* \leq \zeta|\mathbf{X}=\mathbf{x}) + \mathbb{P}(T^*>\zeta|\mathbf{X}=\mathbf{x}) = 1$ or 
\begin{flalign*}
    &&
    \sum_{k=1}^K \underbrace{\mathbb{P}(T^* \leq \zeta \cap \Delta^* = k|\mathbf{X}=\mathbf{x})}_{k^{th} \text{CIF}}+\underbrace{\mathbb{P}(T^*>\zeta|\mathbf{X}=\mathbf{x})}_{\text{Survival Probability}} = 1
    &&
    \text{\small (outputs sum to one)}
\end{flalign*}
With loss presented in Eq.\ref{eqn:loss} we can directly predict the CIF instead of predicting the hazards function (the derivative of the CIF) as often done --\emph{e.g.} DeepHit \citep{lee_deephit_2018} or SurvTRACE \citep{wang_survtrace_2022}. This allows us to drop the constant-hazard hypothesis \citep[][]{yanagisawa2023proper, kvamme_continuous_2019, wang_survtrace_2022, rindt_survival_2022}.

\begin{algorithm}[t]
       \caption{MultiIncidence Algorithm - Training}
       \label{alg:incidence}
    \begin{algorithmic}
       \STATE {\bfseries Input:} $\mathbf{x}, \delta, t$
       \ENSURE $\operatorname{min}(t) >0$
    
       \STATE $\hat{G} \gets$ Train $n_{censoring}$ iterations the censoring algorithm
       \FOR[Boosting iterations]{$m=1$ {\bfseries to} $n_{iter}$}
       \FOR{$i=1$ {\bfseries to} $n_{samples}$}
       \STATE $\zeta_{i} \sim \mathcal{U}(0, t_{max})$ \COMMENT{Sample a random time horizon}
       \STATE $y_i \gets \delta_i ~ (t_i \leq \zeta_{i})$ \COMMENT{Computing the target}
       \STATE $w_i \gets 0$ 
       \IF[The observation is not censored]{$t_i > \zeta_{i}$}
            \STATE $w_i \gets \frac{1}{\hat{G}(\zeta_{i})}$
        \ELSIF{$t_i \leq \zeta_{i}$ and $\delta_i \neq 0$}
            \STATE $w_i \gets \frac{1}{\hat{G}(t_{i})}$
        \ENDIF
        \STATE $\Tilde{\mathbf{x}}_i \gets (\mathbf{x}_i, \zeta_{i})$ \COMMENT{Stacking the time to the features}
       \ENDFOR
       \STATE $\zeta \gets (\zeta_i)_{1 \leq i \leq n_{samples}}$
        \STATE $\hat{H}_m(\mathbf{\tilde{x}}) \gets$ Train one iteration of Gradient Boost with $\hat{G}(\zeta |\mathbf{X}= \mathbf{x})$ \COMMENT{Hm is the m-th tree}
       \STATE $(\hat{S}(\zeta |\mathbf{X}= \mathbf{x}), (\hat{F}_k(\zeta |\mathbf{X}= \mathbf{x})_{1\leq k \leq K}) \gets \hat{H}_m(\mathbf{\tilde{x}})$ 
       \STATE $\hat{G} \gets$ Train one iteration the Censoring Feedback Loop with $\hat{S}(\zeta |\mathbf{X}= \mathbf{x})$
       \ENDFOR
    \end{algorithmic}
\end{algorithm}
Our algorithm uses two classifiers (here gradient-boosted trees), one for the censoring trained on binary censored/non-censored labels (i.e. for time $\zeta$, $\mathbb{P}(C > \zeta| \mathbf{X}= \mathbf{x})$), and a classifier for the multiple events. Both of the censoring and event models are corrected with IPCW weights.
To compute these IPCW we iterate the training using a feedback loop (in the like of boosting). We first compute a survival censoring model. Then, with these probabilities, we initiate our MultiIncidence model. After several iterations, we apply a feedback loop to retrain our censoring model.\\
To model complex time dependence, time is stacked as an additional feature. 
At each iteration, we sample different times for each sample and stack the different features as well as the targets to provide more information to our algorithm. 
This approach is made possible by our loss which is separable.
An additional benefit is that we can predict the CIF at any time, unlike models that are optimized for a limited number of times (such as SurvTRACE) and need to be interpolated to other times. 

As Figure \ref{architecture} shows an iteration: we compute the weights for each sample, as well as the target according to the sampled time. A censored sample will have a weight equal to 0 (due to the indicator functions in eq.\ref{eqn:full_loss}). For strictly positive weights, if the target is in $\llbracket1, K\rrbracket$, this will represent that the event of interest has happened before $\zeta$. Finally, a target equal to 0 will notify that the sample has survived any event. We give a pseudocode of the algorithm \ref{alg:incidence}.

\section{Experimental study: Competing risks}

\subsection{Evaluation metrics for competing risks models}

To evaluate the risks of the different events, we use two
metrics\footnote{We do not focus on the
C-index in time, as this metric is biased \citep{blanche_c-index_2019, rindt_survival_2022}}.

\paragraph{Evaluating the predicted probability} We use a proper scoring
rule (PSR). To avoid a form of circularity in the evaluation, we
do not use the PSR that our model optimizes but rather we 
extend that used by \citet{graf_assessment_1999} and
\citet{schoop_quantifying_2011}: we apply it to the Brier Score rather
than the log-loss (Appendix \ref{sec:evaluation_psr} details the formula and the
formal proof that it is indeed proper). To evaluate the model at
all times, we sum it over time, giving the
\emph{integrated Brier score} (IBS).

\paragraph{Prediction accuracy in time}
For many applications, as in predictive maintenance or medicine, a crucial information is: which is the first event that a
subject may encounter. We use a validation metric to check for each sample whether observed events are predicted as the most likely, at given times, chosen as before with quantiles. \emph{E.g.} for an individual that encounters event 2 at $t$, the probability of surviving before $t$ should be the highest compared to the probabilities of encountering each event. We also want the probability of encountering event 2 after $t$ to be the highest one. To do so, we adapt Multi-Class accuracy to different times: 
\begin{definition}[Prediction accuracy at time $\zeta$]
    For a fixed time horizon $\zeta$ and denoting the survival to any event as the index 0, define $\hat{y} = \argmax\limits_{k \in [0, K]} \hat{F}_{k}(\zeta | \mathbf{X}=  \mathbf{x})$, the most probable event in $\zeta$ and $y_\zeta = \mathbb{1}_{t\leq \zeta} \delta$. We remove the censored individuals and $n_{nc}$ represents the number of individuals uncensored at $\zeta$.
    \begin{equation}
        Acc(\zeta)= \frac{1}{n_{nc}} \sum_{i=1}^n \mathbb{1}_{\hat{y}_i = y_{i, \zeta}}~\mathbb{1}_{\overline{\delta_i = 0, t_i \leq \zeta}}
    \end{equation}
\end{definition}
\subsection{Experimental settings}
\paragraph{Synthetic Dataset}
We designed a synthetic dataset with linear relations between features and targets, as well as relations with the censoring distribution of the features (details in Appendix \ref{synthedata}). 
To create the synthetic dataset, for each sample, we draw $2 n_{events}$ parameters from a normal law. Then, we draw the durations from a Weibull distribution for each event from those parameters. To determine the observation, we return the minimum duration with its associated event. Then, the censoring event is computed with the same method. 
\paragraph{SEER Dataset}
This dataset follows more than 470k breast cancer patients for up to ten years
with mortality due to various diseases as outcomes. The censoring is
around $63\%$ and Figure \ref{fig:seer} shows the distribution of the
events. Instead of \cite{lee_deephit_2018} (DeepHit) or
\cite{wang_survtrace_2022} (SurvTRACE), which consider only the
two most prevalent events and censor the rest, defeating the purpose of competing risks, we consider the SEER data set with 3 competing
events, aggregating the other events in a third class. We remove some
features following \citet{wang_survtrace_2022}.
\paragraph{Baselines}
We compared our approach to 7 other models. Aalen-Johansen's estimator \citep{aalen_survival_2008}, Fine \& Gray's linear model \citep{fine_proportional_1999}, a tree-based approach with the Random Survival Forests \citep[RSF,][]{ishwaran_random_2008}, and neural networks: DeepHit \citep{lee_deephit_2018}, Deep Survival Machines \citep[DSM,][]{nagpal2021deep}, DeSurv \citep{pmlr-v151-danks22a} and a transformer model with SurvTRACE \citep{wang_survtrace_2022}. DeepHit is trained with a ranking loss: the C-index summed with a negative log-likelihood, DSM uses a graphical method for feature encodings while DeSurv solves Ordinal Differential Equations for continuous predictions in time. SurvTRACE is trained for three-time horizons (based on quantiles of observed event times) and at time 0, while Aalen-Johansen and Fine \& Gray are trained for all observed event times. In contrast, our method is trained on uniformly sampled time horizons, allowing predictions at any time. To compute the Integrated Brier Score over time, other methods require linear interpolation of their trained times. For times exceeding their trained times, we assume the incidence remains constant.
To be fair across models, we use the same time budget for hyper-parameter tuning (grid in Appendix \ref{tab:gridsearch}).
\subsection{Results, competing risks}
\paragraph{Synthetic dataset}
Figure \ref{fig:tradeoff_competing_synthetic} shows the trade-off between statistical performance (IBS) and training time for each model compared. With the synthetic dataset, we can compute an oracle IBS. MultiIncidence outperforms the other models over the IBS while being the fastest to train.  

We also conduct different experiments on the synthetic dataset varying the number of training points (\autoref{ipsr}), the censoring rate (\autoref{censo}), and the number of features (\autoref{ipsrvstime}). More experiments on the synthetic data set can be found in the appendix \ref{app:synthetic_results}.

\begin{figure}[b!]%
    \begin{minipage}{.26\linewidth}
    \caption{\textbf{Trade-off prediction/training time for competing
risk on the synthetic dataset} Average IBS compared to the fitting time for each model on 20k training data points, censoring rate around 50\%, and a dependant censoring for 6 features.}
    \label{fig:tradeoff_competing_synthetic}
    \end{minipage}%
    \hfill%
    \begin{minipage}{.72\linewidth}%
    \includegraphics[width=1.02\linewidth]{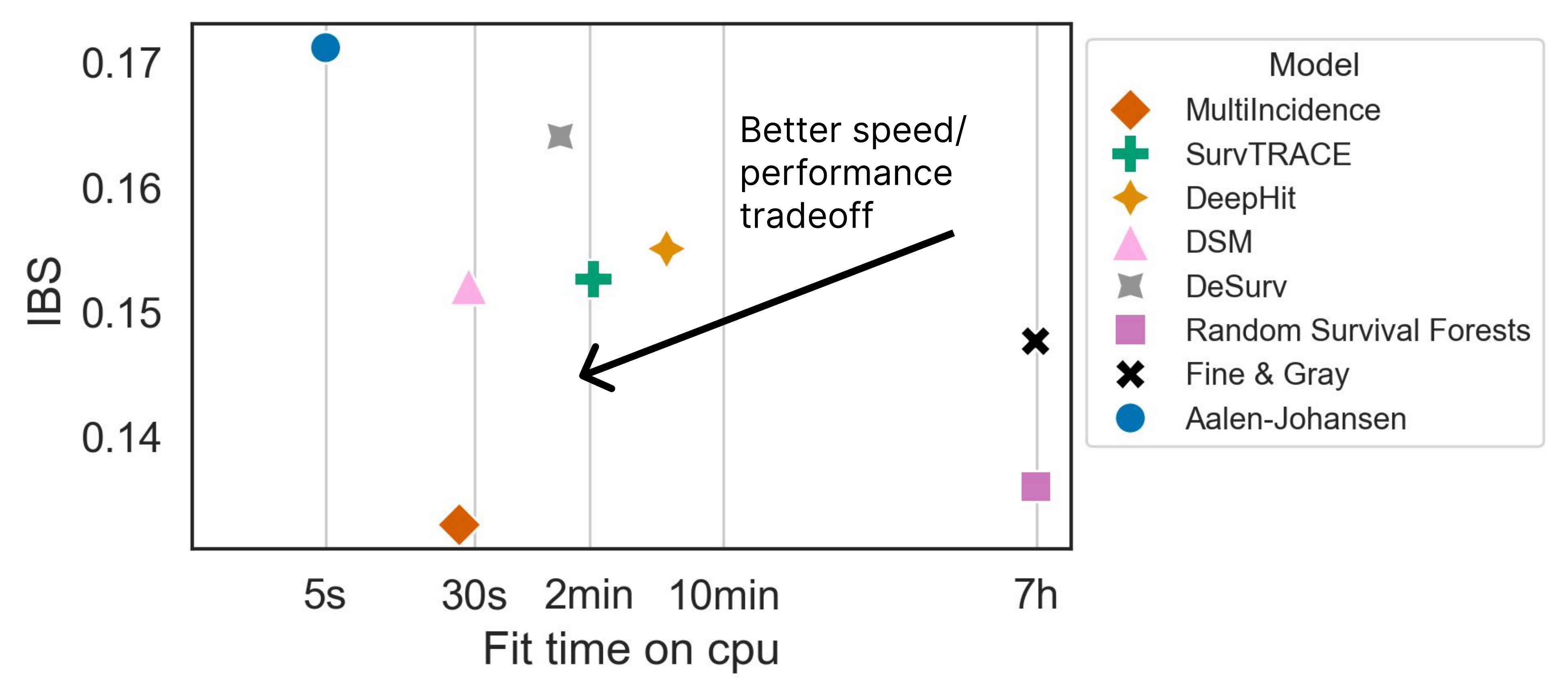}%
    \end{minipage}%
\end{figure} 
\paragraph{Results on SEER Dataset}
On the real-life dataset, we keep 30\% of the data set to test the models. 
\autoref{fig:tradeoff_competing}
compares models with the Integrated Brier score (with Kaplan-Meier weights of
\cite{graf_assessment_1999} due to lack of oracle).
MultiIncidence achieves the best score and the shortest fit time.
Random Survival Forest is not made to be used with that many samples (100k) and uses more than 50\,Gb of RAM.
MultiIncidence maintains its marked lead with much fewer training samples (Appendix \ref{app:seer_results}). 

\begin{figure}[t!]
    \begin{minipage}{.26\linewidth}
    \caption{\textbf{Trade-off prediction/training time for competing
risk on the SEER dataset} Average IBS compared to the fitting time for each model on the maximum training points (330k) except for Fine \& Gray (50k) and RSF (100k). Table \ref{tab:ibs_event_seer} gives IBS values for each event.}
    \label{fig:tradeoff_competing}
    \end{minipage}%
    \hfill%
    \begin{minipage}{.72\linewidth}
    \includegraphics[width=1.02\linewidth]{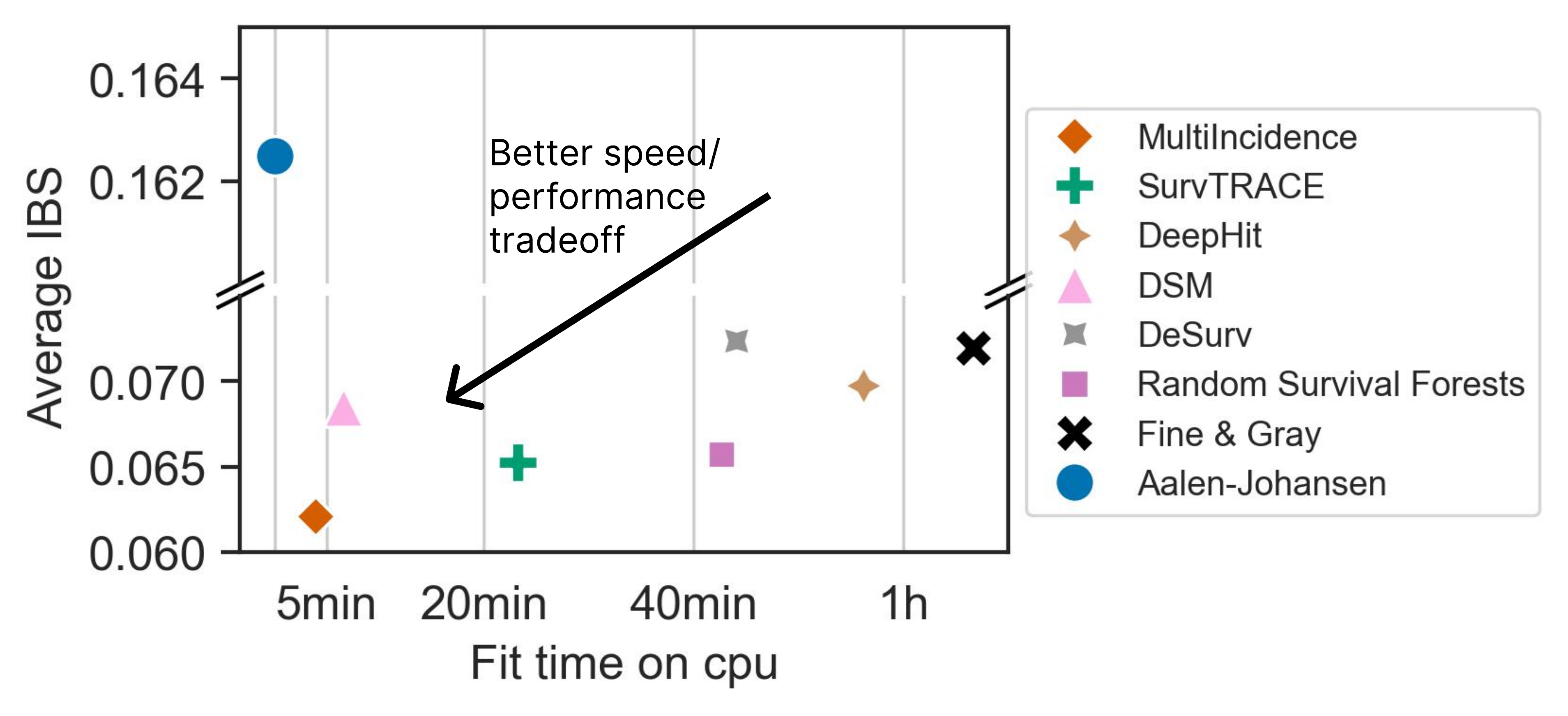}%
    \end{minipage}%
\end{figure} 

Event and time-specific C-indexes are presented in table~\ref{tab:results_seer}, but do not capture the models' ability to predict which event is more likely to occur at a given time horizon. This is measured by accuracy in
time in \autoref{fig:acc}, and MultiIncidence has the best performance. The benefit grows as time increases, meaning that it
better interpolates in times. 
\begin{figure}[t!]
    \begin{minipage}{.26\linewidth}
    \caption{\textbf{Prediction accuracy at time $\zeta$} Accuracy of the Argmax of the Cumulative Incidence Functions on different quantiles in time on the SEER Dataset (Higher is Better).}
    \label{fig:acc}
    \end{minipage}%
    \hfill%
    \begin{minipage}{.72\linewidth}
	\includegraphics[width=1.02\columnwidth]{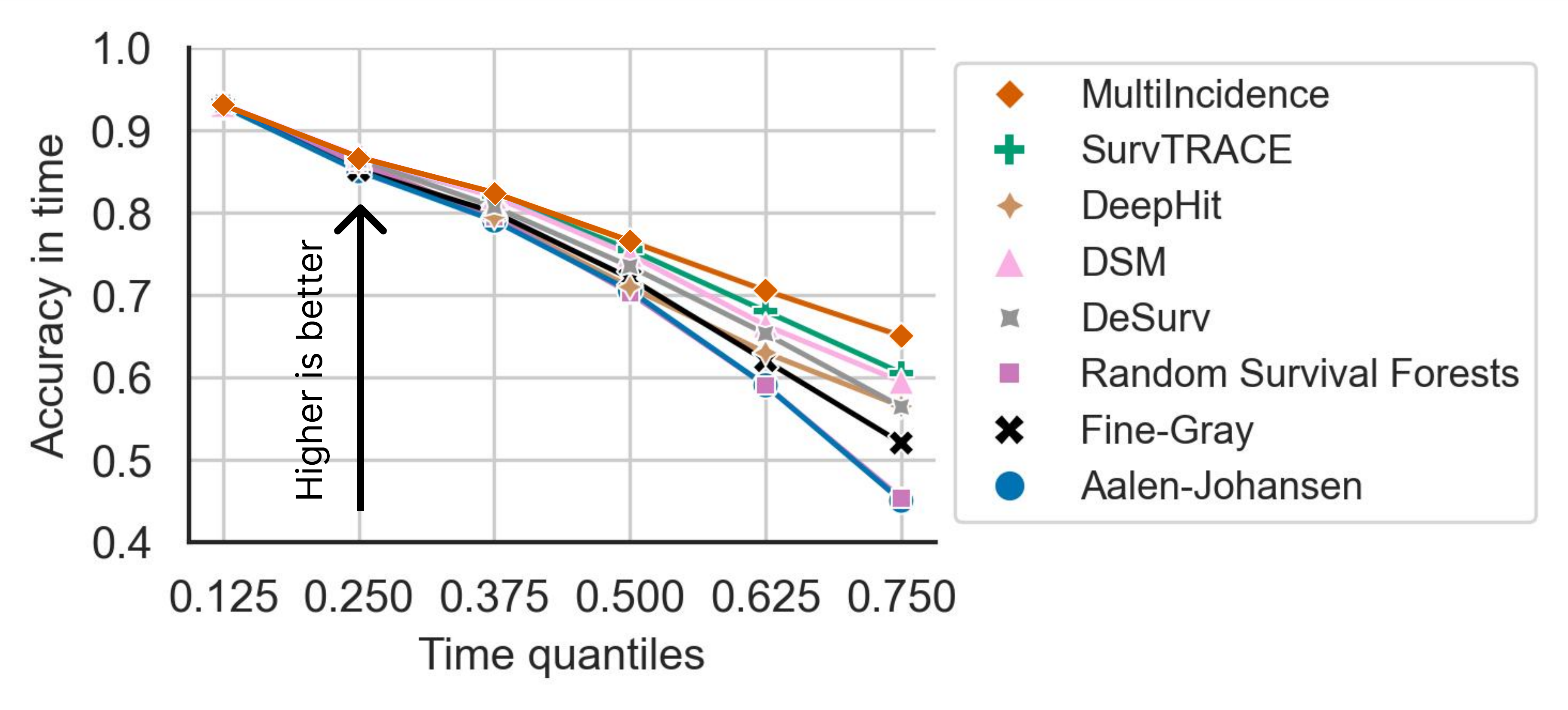}%
    \end{minipage}%
\end{figure}
\section{Usage in Survival Analysis}
\subsection{Survival experiments}
\paragraph{Real-life Datasets}
As our model can also handle survival analysis, we perform survival
analysis on two real-life survival datasets: SUPPORT and METABRIC, both available in the Pycox library. 
\begin{description}[itemsep=1pt, parsep=1pt, topsep=0pt]
\item[METABRIC] The Molecular Taxonomy of Breast Cancer International Consortium is a dataset on gene expression with around 2k data points\item[SUPPORT] Study to Understand Prognoses Preferences Outcomes and Risks of Treatment is a dataset on the survival time of hospital patients with more than 8k datapoints.
\end{description}

\paragraph{Evaluation}
We use different metrics to evaluate our models. As above we use the Integrated Brier Score (detailed in Appendix \ref{sec:evaluation_psr}), but we also add another metric from \citet{yanagisawa2023proper}, called $S_{Cen-log-simple}$ (detailed in Appendix \ref{sec:cen_log_simple}). 
This last metric approximates the proper scoring metric in \citet{rindt_survival_2022} --and is not exactly proper, see  Appendix \ref{sec:cen_log_simple}. It is useful because it can be used on any model as it does not require the density of the Cumulative Incidence Function. 

\paragraph{Baselines}
We compare our model with SOTA competing risks models, including SurvTRACE \citep{wang_survtrace_2022}, DeepHit \citep{lee_deephit_2018} and Random Survival Forests \citep{ishwaran_random_2008}.
We also benchmark some SOTA survival ones: neural networks \emph{e.g.} \citep[PCHazard][]{kvamme_continuous_2019}, survival game \citep{han2021inverse} and neural networks trained with a proper survival-analysis scoring rule, \emph{e.g.} SumoNet \citep{rindt_survival_2022}, and DQS \citep{yanagisawa2023proper}.

\subsection{Results in survival usage}

\begin{table}[t] 
\caption{\textbf{Survival dataset}: Integrated Brier Score and $S_{Cen-log-simple}$ (Lower is Better)
\label{tab:survival}}
\small\sc
\begin{tabular}{l|rr|rr}
\toprule
\hfill Dataset & \multicolumn{2}{c|}{\textbf{SUPPORT}} & \multicolumn{2}{c}{\textbf{METABRIC}} \\
\toprule
Model& IBS& $S_{Cen-log-simple}$ & IBS & $S_{Cen-log-simple}$ \\
\midrule
Random Survival Forest &0.225±0.004&1.942±0.023&0.197±0.025&2.442±0.044\\
DeepHit & 0.217±0.004 & 2.249±0.009 &0.180±0.014&2.271±0.019\\
\citet{han2021inverse} &0.260±0.012&3.483±0.307 &0.191±0.003&2.420±0.150 \\
PCHazard&0.210±0.007&2.192±0.024&0.176±0.014&2.246±0.046\\
Han&0.260±0.012&3.483±0.307&0.191±0.003&2.420±0.150\\
DQS&0.202±0.007&1.987±0.069&0.180±0.034&2.205±0.044\\
SuMo net&\underline{0.194±0.010}&\textbf{1.721±0.016}&\underline{0.169±0.009}&\textbf{2.302±0.059}\\
SurvTRACE&\underline{0.194±0.005}&1.870±0.018&\textbf{0.168±0.011}&2.270±0.034\\
MultiIncidence&\textbf{0.191±0.006}& \underline{1.740±0.020}&\textbf{0.168±0.019}&\textbf{2.169±0.056}\\
\bottomrule
\end{tabular}
\end{table}

\paragraph{Prediction performance} For both datasets, MultiIncidence achieves the best results on IBS and tied with Sumo Net for $S_{Cen-log-simple}$ (Table \ref{tab:survival} and Appendix \ref{saall} for the C-index). Sumo Net uses $S_{Cen-log-simple}$ as a training loss; note however that this metric is not guaranteed to be a proper scoring rule thus it does not ensure recovering the actual risks.

\begin{figure}[t]
    \begin{minipage}{.26\linewidth}
    \caption{\textbf{Trade-off prediction/training time in survival
usage} Performances (measured by IBS, integrated Brier score) function of fitting time for each model.}
    \label{fig:tradeoff}
    \end{minipage}%
    \hfill%
    \begin{minipage}{.72\linewidth}
    \includegraphics[width=0.9\columnwidth]{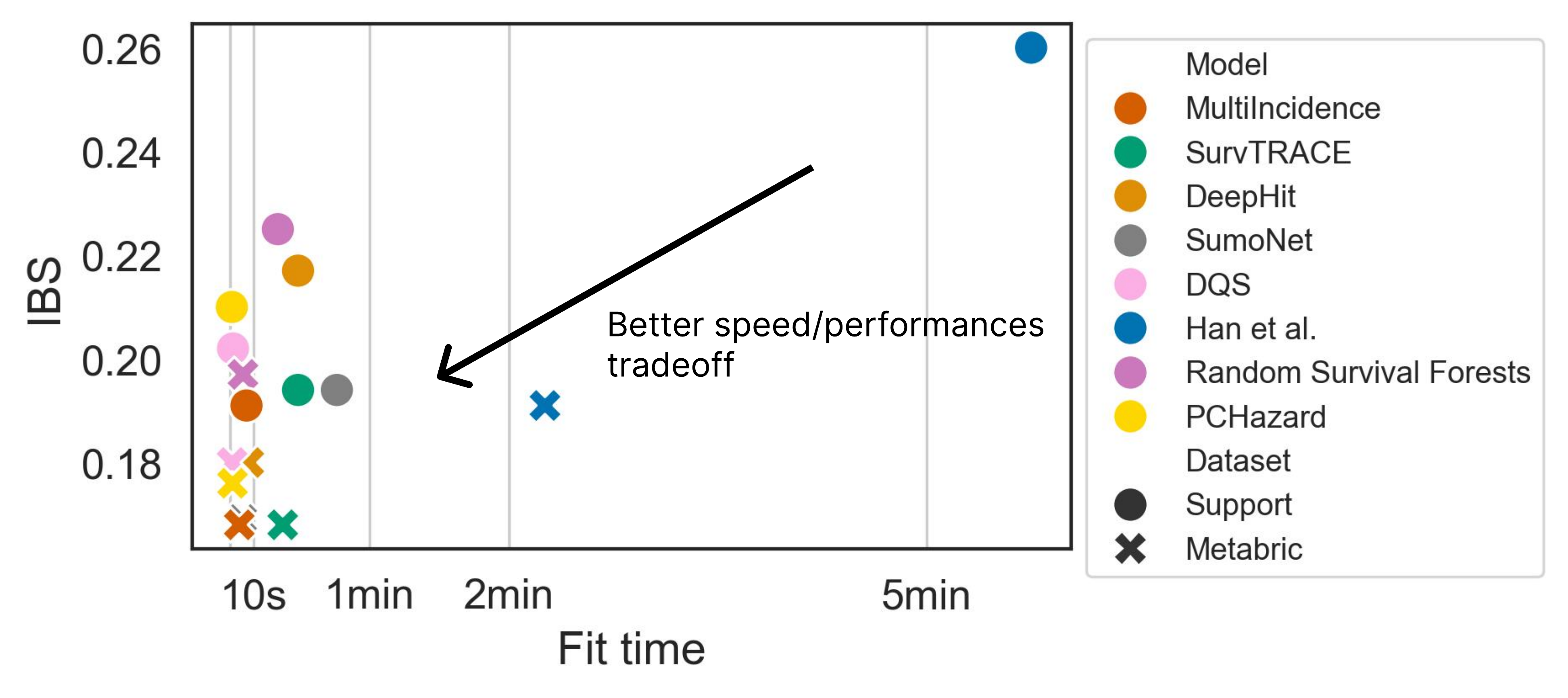}
    \end{minipage}%
\end{figure}

\paragraph{Computational time}
Figure \ref{fig:tradeoff} shows the trade-off between training time and performance in IBS, a trade-off that MultiIncidence excels at, being the best model for statistical performance and also one of the fastest. Appendix \ref{tradeoffyana} gives the same figure for the $S_{Cen-log-simple}$ metric, and MultiIncidence reaches a great trade-off rivaled only by SumoNet, which has competing performance on the $S_{Cen-log-simple}$ loss. Varying sample size from 1k to 100k on a synthetic dataset confirms that MultiIncidence and DQS are faster (less than 1min on 100k data points), \citeauthor{han2021inverse}, SumoNet, and Random Survival Forests slower for large sample size, with a super-linear time complexity for SumoNet and Random Survival Forests that makes them untractable for large data (Appendix \ref{survtime}).  

\section*{Discussion and Conclusion}

\paragraph{Code reproducibility and data}
The code is available on GitHub as a library called \hyperlink{https://github.com/soda-inria/hazardous/tree/main}{hazardous}. 

\paragraph{Social impact} Our contribution is not directly applied and has no immediate social impact, but we hope that it will improve medical applications where survival analysis is central.

\paragraph{Limitations and further work}

Further work should consider removing the assumption of noninformative censoring. This assumption is very common in the literature, though some recent works have relaxed it in survival settings \citep{foomani2023copulabased,zhang2023deep}.

\paragraph{Conclusion}

For competing risks, which is a generalization of survival analysis to classify the type of outcome, we first propose and prove a (strictly) proper scoring rule. It is a reweighted log loss that can easily be used as a loss for machine learning: it is separable in the observations and thus suited to stochastic solvers; it does not require time-wise derivative (unlike most survival models) and can be used in non-differentiable models. We plug it into gradient-boosting trees, in an algorithm called MultiIncidence. Thanks to time used as a feature and its feedback loop to better estimate censoring probabilities, MultiIncidence outperforms state-of-the-art methods on a synthetic dataset as well as real-life datasets both for competing risk (classification on time-censored data) and standard survival (time-to-event regression with right censoring). It is also faster to train over many samples. As a loss, it easily brings survival or competing risks to many models: scalable linear models to replace clinical standard \citeauthor{fine_proportional_1999} that do not scale, or deep learning, including fine-tuning foundation models.

\bibliography{papier2}
\bibliographystyle{unsrtnat}

\newpage

\newpage
\appendix

 \xdef\presupfigures{\arabic{figure}}
  \xdef\presuptables{\arabic{table}}
\renewcommand\thefigure{S\fpeval{\arabic{figure}-\presupfigures}} 
\renewcommand\thetable{S\fpeval{\arabic{table}-\presuptables}} 

\onecolumn

\section{Definitions}
\subsection{Notations}

Here we detail the notations used in the main manuscript as well as in the proofs and derivations below.

For all symbols, we use the following conventions:
\begin{itemize}
    \item $.^*$: Oracle
    \item $\hat{.}$: Estimation
\end{itemize}

The different variables that we use are:
\begin{table}[ht]
\begin{tabularx}{\linewidth}{r r X}
\toprule
Maths Symbol & Domain & Description\\
\midrule 
    $\zeta$ & $\mathbb{R}_+$& Time horizon \\
\midrule
    $K$ & $\mathbb{N}^*$ & number of competing events (events of interest)\\ 
    $\mathbf{X}$ & $\mathcal{X}$& random variable representing an individual \\
    $T^*_k$ &$\mathbb{R}_+$& random variable when the event $k$ will occur\\
    $C$ &$\mathbb{R}_+$& random variable when the censoring will occur\\
    $T^*$ & $\mathbb{R}_+$ & $\min (T^*_1, T^*_2, ..., T^*_K)$\\
    $T$ & $\mathbb{R}_+$ & $\min (T, C)$ \\
    $\Delta^*$ &$[1, K]$&$\argmin\limits_{i \in [1, K]} (T^*_i)$\\
    $\Delta$ &$[0, K]$&$\argmin (C, T^*_1, T^*_2, ..., T^*_K)$\\
    \midrule
    S & $\mathcal{S}$ & Survival function\\
    F & $\mathcal{F}$ & Cumulative Incidence Function \\
    G & $\mathcal{G}$ & Censor function\\
    \midrule
    $n$ &$\mathbb{N}^*$& number of individuals in our observation \\
    $i$ &$[1, n]$& one observation\\
    $\mathbf{x}_i$ &$\mathcal{X}^n$& individuals observed \\
    $t_i$ &$\mathbb{R}_+^n$& time-to-event/censoring observed\\
    $\delta_i$ & $[0, K]$& event observed, 0 means censoring \\
\bottomrule
\end{tabularx}
\caption{Notations used}
\label{tab:history}
\end{table}

\subsection{Reporting conventions}

In tables, the best results are reported in bold characters, and the second best is underlined.
\section{Theory on our proper scoring rule: proofs and derivations}
In this appendix, we give the proofs and derivations concerning the proper scoring rule that we have introduced.

\begin{proof}[Proof the of Lemma \ref{lem:usefulllemma} on the expectation of the Reweighted NLL]\label{prooflemma}
\usefullemma*

\begin{multline}
    \forall \zeta, \forall k \in \llbracket1, K\rrbracket, (\mathbf{x}, t, \delta) \sim \mathcal{D} \\
    \mathrm{L}_{\zeta}(\hat{F}(\zeta| \mathbf{x}), (t, \delta)) \defeq \frac{1}{n} \sum_{i=1}^n \left(\sum_{k=1}^{K}
    \dfrac{
        \mathbb{1}_{t_i \leq \zeta, \delta_i = k} ~~\log\left(\hat{F}_k(\zeta|\mathbf{x}_i)\right)
        }
        {
        \tikzmarknode{censure_i_eval}{\highlight{orange}
            {$G^*(t_i|\mathbf{x}_i) $}
            }} \right)
        +
        \dfrac{
        \mathbb{1}_{t_i > \zeta} ~~ \log\left(\hat{S}(\zeta|\mathbf{x}_i)\right)
        }
        {
        \tikzmarknode{censure_t_eval}{\highlight{cyan}
            {$G^*(\zeta|\mathbf{x}_i) $}
        }}
\end{multline}

\paragraph{Derivation of the expectation}:
\begin{align*}
    \mathbb{E}_{T, \Delta |\mathbf{X} =\mathbf{x}}\left[\Psi_{k, \zeta}(\hat{F}_k(\zeta| \mathbf{x}), ( T, \Delta))|\mathbf{X}=\mathbf{x}\right] &= 
    \mathbb{E}_{T, \Delta |\mathbf{X} =\mathbf{x}}\left[ 
                    \mathbb{1}_{T \leq \zeta} \mathbb{1}_{\Delta = k}  
                    \dfrac{\log\left(\hat{F}_k(\zeta|\mathbf{x})\right)}{G^*(T|\mathbf{x})}  
                    \right] \\
    &=\log\left(\hat{F}_k(\zeta|\mathbf{x})\right) ~\mathbb{E}_{T, \Delta |\mathbf{X} =\mathbf{x}}\left[
                    \dfrac{
                    \mathbb{1}_{T^* \leq \zeta} \mathbb{1}_{T^* \leq C^*} 
                    \mathbb{1}_{\Delta = k} 
                    }
                    {
                    G^*(T|\mathbf{x})
                    }
                    \right] \\
    &= \log\left(\hat{F}_k(\zeta|\mathbf{x})\right)
        \mathbb{P}(T^*\leq \zeta, \Delta = k | \mathrm{X}= \mathrm{x}) \\
\end{align*}
And: 
\begin{align*}
    \mathbb{E}_{T, \Delta |\mathbf{X} =\mathbf{x}}\left[\Lambda_{k, \zeta}(\hat{S}(\zeta| \mathbf{X}= \mathbf{x}), (T, \Delta))|\mathbf{X}=\mathbf{x}\right] &= 
    \mathbb{E}_{T, \Delta |\mathbf{X} =\mathbf{x}}\left[
                    \mathbb{1}_{T>\zeta}
                    \dfrac{\log\left(\hat{S}(\zeta| \mathbf{X}= \mathbf{x})\right)}{G^*(\zeta|\mathbf{x})}  
                    \right] \\
        &= \log\left(\hat{S}(\zeta| \mathbf{X}= \mathbf{x})\right) \mathbb{E}_{T, \Delta |\mathbf{X} =\mathbf{x}}\left[ 
                    \dfrac{
                    \mathbb{1}_{T^*>\zeta}
                    \mathbb{1}_{C^*>\zeta}
                    }
                    {\mathbb{P}(C^*>\zeta|\mathbf{x})} 
                    \right]\\
        &= \log\left(\hat{S}(\zeta| \mathbf{X}= \mathbf{x})\right) \mathbb{P}(T^*>\zeta| \mathbf{X} = \mathbf{x})\\
\end{align*}

By summing all of the terms, we obtain: 

\begin{align}
    \mathbb{E}_{T, \Delta |\mathbf{X} =\mathbf{x}}\left[\mathrm{L}_{ \zeta}\left(\hat{F}(\zeta|\mathbf{x}), ( T, \Delta)\right)\right] &= 
    \sum_{k=1}^K \log\left(\hat{F}_k(\zeta|\mathbf{x}), 1\right)
        \mathbb{P}(T^*\leq \zeta, \Delta = k)  
     + \log\left(\hat{S}(\zeta| \mathbf{X}= \mathbf{x})\right)
        \left(\mathbb{P}(T^*>\zeta| \mathbf{X} = \mathbf{x})\right) \\
    &=\sum_{k=1}^K \log\left(\hat{F}_k(\zeta|\mathbf{x})\right)
        F^*_k(\zeta |\mathrm{x})  
     + \log\left(\hat{S}(\zeta|\mathbf{x})\right)
        S^*(\zeta| \mathbf{x})
\end{align}

Finally:
\begin{align}
    \mathbb{E}_{T, \Delta |\mathbf{X} =\mathbf{x}}\left[\mathrm{L}_{ \zeta}\left(\hat{F}(\zeta|\mathbf{x}), ( T, \Delta)\right)\right]
    &=\sum_{k=1}^K \log\left(\hat{F}_k(\zeta|\mathbf{x})\right)
        F^*_k(\zeta |\mathrm{x})  
     + \log\left(\hat{S}(\zeta|\mathbf{x})\right)
        S^*(\zeta| \mathbf{x})
\end{align}

\end{proof}

\begin{proof}[Proof of the Theorem \ref{thm:bigtheorem}]\label{psrweights}
\bigthm*
To be more explicit, we can define a new random variable $Y$: 
\begin{definition}
     $$\forall \zeta, ~ Y_{k, \zeta} \defeq T^* \leq \zeta \cap \Delta = k$$
     And:
     $$\forall \zeta, ~ Y_{0, \zeta} \defeq T^* > \zeta$$
\end{definition}  

\begin{equation}
    F^*_k(\zeta |\mathbf{x}) = \mathbb{P}(T^*\leq \zeta, \Delta = k| \mathbf{X} = \mathbf{x}) = \mathbb{P}(Y_{k, \zeta}=1| \mathbf{X} = \mathbf{x})
\end{equation}
\begin{equation}
    S^*(\zeta |\mathbf{x}) = \mathbb{P}(T^*> \zeta| \mathbf{X} = \mathbf{x}) = \mathbb{P}(Y_{0, \zeta}=1| \mathbf{X} = \mathbf{x})
\end{equation}

$\hat{F}_k(\zeta|\mathbf{x})$ represents the estimated probability for $Y_{k, \zeta}=1$,
so we can rewrite: $\hat{p}_{k, \zeta} \defeq \hat{F}_k(\zeta|\mathbf{x}) \approx \mathbb{P}(Y_{k, \zeta}=1| \mathbf{X} = \mathbf{x})$
Therefore:

\begin{flalign}
    \mathbb{E}_{T, \Delta |\mathbf{X} =\mathbf{x}}\left[\mathrm{L}_{k, \zeta}(\hat{F}_k(\zeta |\mathbf{x}), ( T, \Delta))\right] &=\mathbb{E}_{T, \Delta |\mathbf{X} =\mathbf{x}}[\mathrm{L}_{\zeta}(\hat{p}_{\zeta}, ( T, \Delta))] \\
    &= \sum_{k=1}^K \log\left(\hat{p}_{k, \zeta}\right)
        \mathbb{P}(Y_{k, \zeta}=1| \mathbf{X} = \mathbf{x})  
     + \log\left(\hat{p}_{0, \zeta}\right)
        \mathbb{P}(Y_{0, \zeta}=0| \mathbf{X} = \mathbf{x})
\end{flalign}
Thus:
\begin{equation}
\begin{aligned}
\min_{\hat{p}} \quad & \sum_{k=1}^K \log\left(\hat{p}_{k, \zeta}\right)
        \mathbb{P}(Y_{k, \zeta}=1| \mathbf{X} = \mathbf{x})  
     + \log\left(\hat{p}_{0, \zeta}\right)
        \mathbb{P}(Y_{0, \zeta}=1| \mathbf{X} = \mathbf{x})\\
\textrm{s.t.} \quad & \sum_{k=0}^K \hat{p}_k = 1\\
  &\hat{p}_k\geq0    \\
\end{aligned}
\end{equation}
is obtained for $\hat{p}= p^*$.
\begin{align*}
    \frac{p_k}{\hat{p}_k}  + \mu &=0 \\
    \sum_{k=0}^K\hat{p}_k &= 1\\
    \sum_{k=0}^K\frac{p_k}{-\mu}&= 1\\
    \mu &= -1\\
    \Longrightarrow \hat{p}_k &= a_k\\
\end{align*}

\end{proof}

\section{Study of the proper scoring rule used for evaluation}%
\label{sec:evaluation_psr}

As mentioned above, the metric most used in the competing risks setting,
the C-index in time, is biased \citep{blanche_c-index_2019, rindt_survival_2022}. To overcome this issue, which is major for any evaluation strategy, we propose here two evaluation metrics: one re-weighting proper scoring rule, that can be effective with any proper binary scoring rule. The second is the accuracy in time that measures the observed event versus the most likely predicted event.

\subsection{PSR for evaluation}
The PSR introduced in the main paper to be the loss of our model is a global loss over all of our predictions. The following loss is adapted to focus on a special event $k$ to evaluate our estimations on a specific event. In the paper, we chose to focus on the IBS, but one could use a logarithmic loss because of its properness.  
\paragraph{Proper scoring rule for the $k^{th}$ competing event}
In our setting, we will denote $L_{k, \zeta}$, a scoring rule for the $k^{th}$ CIF at a time horizon $\zeta$. 
\begin{definition}[\emph{PSR for the $k^{th}$ cause-specific event}]
The scoring rule $L_{k, \zeta}$ for the $k^{th}$ CIF at time $\zeta$ for an observation $(\mathbf{X}, T, \Delta)$ is proper if and only if:
 \begin{flalign}
   \forall \zeta, (\mathbf{X}, T, \Delta )\sim \mathcal{D}, ~~
    \mathbb{E}_{T, \Delta  | \mathbf{X}= \mathbf{x}}[L_{k, \zeta}(\hat{F}_k(\zeta| \mathbf{x}), (T, \Delta))]  \leq \mathbb{E}_{T, \Delta  | \mathbf{X}= \mathbf{x}}[L_{k, \zeta}(F^*_k(\zeta| \mathbf{x}), (T, \Delta))]
\end{flalign}
\end{definition}

\subsubsection{A Proper Scoring Rule for Competing Risks}
To evaluate our model, we used the following proper scoring rule is adequate for each event. Thanks to this proper scoring rule, we can understand the error for each event and the global error of all of the CIF. 

In the following, we prove that any given (strictly) proper scoring rule that can be used in the multiclass setting (\emph{e.g.} the Brier score, the negative log-likelihood) leads to a (strictly) proper scoring in competing risks settings thanks to the re-weighting of the observations. \\
Indeed, for any (strictly) proper scoring rule $\ell: \mathbb{R} \times \{0,1\} \rightarrow \mathbb{R}$, one can build a cause-specific scoring rule function $L_{k, \zeta}:  \mathbb{R} \times \mathcal{D} \rightarrow \mathbb{R}$ that is also a (strictly) proper scoring rule for the cause-specific event $k^{th}$ in the fixed time horizon $\zeta \in \mathbb{R}_+$. It follows that $L_{\zeta}$ is (strictly) proper. 
\begin{definition}[\emph{PSR with re-weighting}]
We define $L_{k,\zeta}$, considering the observations $(\mathbf{x}, t, \delta)$ and for an event $k$, the following scoring rule of the $k^{th}$ CIF:
\begin{multline}
    \forall \zeta, \forall k \in \llbracket1, K\rrbracket, \ell: \mathbb{R} \times \{0,1\}\rightarrow \mathbb{R}, (\mathbf{x}, t, \delta) \sim \mathcal{D} \\
    \mathrm{L}_{k, \zeta}(\hat{F}_k(\zeta| \mathbf{x}), (t, \delta)) \defeq \frac{1}{n} \sum_{i=1}^n 
    \dfrac{
        \mathbb{1}_{t_i \leq \zeta, \delta_i = k} ~~\ell\left(\hat{F}_k(\zeta|\mathbf{x}_i), 1\right)
        }
        {
        \tikzmarknode{censure_ie1}{\highlight{orange}
            {$G^*(t_i|\mathbf{x}_i) $}
            }} \\
        +
        \dfrac{
        \mathbb{1}_{t_i > \zeta} ~~ \ell\left(\hat{F}_k(\zeta|\mathbf{x}_i), 0\right)
        }
        {
        \tikzmarknode{censure_te}{\highlight{cyan}
            {$G^*(\zeta|\mathbf{x}_i) $}
        }} \\
    + 
        \dfrac{
        \mathbb{1}_{t_i \leq \zeta, \delta_i \neq 0, \delta_i \neq k} ~~ \ell\left(  \hat{F}_k(\zeta|\mathbf{x}_i), 0\right)
        }
        {
        \tikzmarknode{censure_ie}{\highlight{orange}
            {$G^*(t_i|\mathbf{x}_i) $}
            }} 
\end{multline}%
\begin{tikzpicture}[overlay,remember picture,>=stealth,nodes={align=left,inner ysep=1pt},<-]
     \path (censure_ie.south) ++ (-10em, 1em) node[anchor=east,color=orange!67] (titlecensi){\text{Probability of remaining at $t_i$}};
     \draw [color=orange!87](censure_ie.west) -- ([xshift=-0.1ex,color=orange]titlecensi.east);
    \path (censure_te.south) ++ (-8em,1em) node[anchor=east,color=cyan!67] (censi){\text{\parbox{25ex}{Probability  of remaining at $\zeta$ \\ \small (1 - probability of  censoring)}}};
     \draw [color=cyan!87](censure_te.west) -- ([xshift=-0.1ex,color=cyan]censi.east);
\end{tikzpicture}
The weights correspond to the Inverse Probability of Censoring Weighting
(IPCW) used to re-calibrate the observed population to align with the
uncensored oracle population \cite{robins_estimation_1994}. This PSR is
an extension of \citet{graf_assessment_1999} and
\citet{schoop_quantifying_2011} when $\ell$ is the Brier Score.
\end{definition} 

\begin{lemma}\label{expectation}
    Considering a proper scoring rule $\ell:\mathbb{R} \times \{0,1\}$, at time horizon $\zeta$ and for any cause-specific risk $k$, the expectation of the former scoring rule can be written as: 
    \begin{multline}
        \forall \zeta, \forall k \in \llbracket1, K\rrbracket, \ell: \mathbb{R} \times \{0,1\}\rightarrow \mathbb{R}, (\mathbf{X}, T, \Delta) \sim \mathcal{D}, \\
        \mathbb{E}_{T, \Delta |\mathbf{X} =\mathbf{x}}\left[\mathrm{L}_{k, \zeta}\left(\hat{F}_k(\zeta|\mathbf{x}), (T, \Delta)\right)\right] = 
    \ell\left(\hat{F}_k(\zeta|\mathbf{x}), 1\right)
        F^*_k(\zeta|\mathbf{x}) 
     + \ell\left(\hat{F}_k(\zeta|\mathbf{x}), 0\right)
        \left(1 - F^*_k(\zeta|\mathbf{x})\right)
    \end{multline}
\end{lemma} 

\begin{proof}
    \begin{multline}
    \forall \zeta, \forall k \in \llbracket1, K\rrbracket, \ell: \mathbb{R} \times \{0,1\}\rightarrow \mathbb{R}, (\mathbf{x}, t, \delta) \sim \mathcal{D} \\
    \mathrm{L}_{k, \zeta}(\hat{F}_k(\zeta| \mathbf{x}), (t, \delta)) \defeq \frac{1}{n} \sum_{i=1}^n 
    \underbrace{
    \dfrac{
        \mathbb{1}_{t_i \leq \zeta, \delta_i = k} ~~\ell\left(\hat{F}_k(\zeta|\mathbf{x}_i), 1\right)
        }
        {G^*(t_i|\mathbf{x}_i)}
    }_{\defeq \Psi_{k, \zeta}(\hat{F}_k(\zeta| \mathbf{x}), ( t, \delta))} \\
        +
    \underbrace{
        \dfrac{
        \mathbb{1}_{t_i > \zeta} ~~ \ell\left(\hat{F}_k(\zeta|\mathbf{x}_i), 0\right)
        }
        {G^*(\zeta|\mathbf{x}_i)} 
    }_{\defeq \Lambda_{k, \zeta}(\hat{F}_k(\zeta| \mathbf{x}), ( t, \delta))} \\
    + 
    \underbrace{
        \dfrac{
        \mathbb{1}_{t_i \leq \zeta, \delta_i \neq 0, \delta_i \neq k} ~~ \ell\left(  \hat{F}_k(\zeta|\mathbf{x}_i), 0\right)
        }
        {G^*(t_i|\mathbf{x}_i)} 
    }_{\defeq \Phi_{k, \zeta}(\hat{F}_k(\zeta| \mathbf{x}), ( t, \delta))} 
\end{multline}
\begin{align*}
    \mathbb{E}_{T, \Delta |\mathbf{X} =\mathbf{x}}\left[\Psi_{k, \zeta}(\hat{F}_k(\zeta| \mathbf{x}), ( T, \Delta))|\mathbf{X}=\mathbf{x}\right] &= 
    \mathbb{E}_{T, \Delta |\mathbf{X} =\mathbf{x}}\left[ 
                    \mathbb{1}_{T \leq \zeta} \mathbb{1}_{\Delta = k}  
                    \dfrac{\ell\left(\hat{F}_k(\zeta|\mathbf{x}), 1\right)}{G^*(T|\mathbf{x})}  
                    \right] \\
    &=\ell\left(\hat{F}_k(\zeta|\mathbf{x}), 1\right) ~\mathbb{E}_{T, \Delta |\mathbf{X} =\mathbf{x}}\left[
                    \dfrac{
                    \mathbb{1}_{T^* \leq \zeta} \mathbb{1}_{T^* \leq C^*} 
                    \mathbb{1}_{\Delta = k} 
                    }
                    {
                    G^*(T|\mathbf{x})
                    }
                    \right] \\
    &= \ell\left(\hat{F}_k(\zeta|\mathbf{x}), 1\right)
        \mathbb{P}(T^*\leq \zeta, \Delta = k | \mathrm{X}= \mathrm{x}) \\
\end{align*}

\begin{align*}
    \mathbb{E}_{T, \Delta |\mathbf{X} =\mathbf{x}}\left[\Phi_{k, \zeta}\left(\hat{F}_k(\zeta| \mathbf{x}), ( T, \Delta)\right)\right] &= 
    \mathbb{E}_{T, \Delta |\mathbf{X} =\mathbf{x}}\left[ 
                    \mathbb{1}_{T\leq \zeta, \Delta \neq 0, \Delta \neq k}
                    \dfrac{\ell\left(\hat{F}_k(\zeta|\mathbf{x}), 0\right)}{G^*(T|\mathbf{x})}  
                    \right] \\
    &=\ell\left(\hat{F}_k(\zeta|\mathbf{x}), 0\right) ~\mathbb{E}_{T, \Delta |\mathbf{X} =\mathbf{x}}\left[ 
                    \dfrac{
                    \mathbb{1}_{T^* \leq \zeta} \mathbb{1}_{T^* \leq C^*} 
                    \mathbb{1}_{\Delta \neq k} 
                    }
                    {
                    G^*(T|\mathbf{x})
                    }
                    \right] \\
    &= \ell\left(\hat{F}_k(\zeta|\mathbf{x}), 0\right)
        \mathbb{P}(T^*\leq \zeta, \Delta \neq k | \mathrm{X}= \mathrm{x})\\
\end{align*}

\begin{align*}
    \mathbb{E}_{T, \Delta |\mathbf{X} =\mathbf{x}}\left[\Lambda_{k, \zeta}(\hat{F}_k(\zeta, \mathbf{x}), (T, \Delta))|\mathbf{X}=\mathbf{x}\right] &= 
    \mathbb{E}_{T, \Delta |\mathbf{X} =\mathbf{x}}\left[
                    \mathbb{1}_{T>\zeta}
                    \dfrac{\ell\left(1 - \hat{F}_k(\zeta|\mathbf{x}), 0\right)}{G^*(\zeta|\mathbf{x})}  
                    \right] \\
        &= \ell\left(\hat{F}_k(\zeta|\mathbf{x}), 0\right) \mathbb{E}_{T, \Delta |\mathbf{X} =\mathbf{x}}\left[] 
                    \dfrac{
                    \mathbb{1}_{T^*>\zeta}
                    \mathbb{1}_{C^*>\zeta}
                    }
                    {\mathbb{P}(C^*>\zeta|\mathbf{x})} 
                    \right]\\
        &= \ell\left(\hat{F}_k(\zeta|\mathbf{x}), 0\right) \mathbb{P}(T^*>\zeta| \mathbf{X} = \mathbf{x})\\
\end{align*}

By summing all of the terms, we obtain: 
\begin{equation}
\begin{multlined}
    \mathbb{E}_{T, \Delta |\mathbf{X} =\mathbf{x}}\left[\mathrm{L}_{k, \zeta}\left(\hat{F}_k(\zeta|\mathbf{x}), ( T, \Delta)\right)\right] = 
    \ell\left(\hat{F}_k(\zeta|\mathbf{x}), 1\right)
        \mathbb{P}(T^*\leq \zeta, \Delta = k) \\ 
     + \ell\left(\hat{F}_k(\zeta|\mathbf{x}), 0\right)
        \left( \mathbb{P}(T^*\leq \zeta, \Delta \neq k | \mathrm{X}= \mathrm{x})+ \mathbb{P}(T^*>\zeta| \mathbf{X} = \mathbf{x})\right)
\end{multlined}
\end{equation}

Meanwhile, 
\begin{flalign}
    \mathbb{P}(\overline{T^* \leq \zeta \cap \Delta = k}) &= \mathbb{P}(T^* > \zeta \cup \Delta \neq k) \\
    &= \mathbb{P}(T^* > \zeta) + \mathbb{P}(\Delta \neq k) - \mathbb{P}(T^* > \zeta \cap \Delta \neq k) \\ 
    &= \mathbb{P}(T^* > \zeta) + \mathbb{P}(\Delta \neq k \cap T^* > \zeta) + \mathbb{P}(\Delta \neq k \cap T^* \leq \zeta) - \mathbb{P}(T^* > \zeta \cap \Delta \neq k) \\ 
    &= \mathbb{P}(T^* > \zeta) + \mathbb{P}(\Delta \neq k \cap T^* \leq \zeta)
\end{flalign}
So, we obtain:
\begin{equation}
    \mathbb{E}_{T, \Delta |\mathbf{X} =\mathbf{x}}\left[\mathrm{L}_{k, \zeta}\left(\hat{F}_k(\zeta|\mathbf{x}), ( T, \Delta)\right)\right]= 
    \ell\left(\hat{F}_k(\zeta|\mathbf{x}), 1\right)
        F^*_k(\zeta|\mathbf{x})
     + \ell\left(\hat{F}_k(\zeta|\mathbf{x}), 0\right)
        \left(1 - F^*_k(\zeta|\mathbf{x})\right)
\end{equation}
\end{proof}

\begin{proposition}\label{psr}
    If $\ell: \mathbb{R} \times \{0,1\} \rightarrow \mathbb{R}$, a chosen (strictly) proper scoring rule, then $L_{k, \zeta}:  \mathbb{R} \times \mathcal{D} \rightarrow \mathbb{R}$ is a (strictly) proper scoring rule for the cause-specific event $k^{th}$ in the fixed time horizon $\zeta \in \mathbb{R}_+$.
\end{proposition}

\begin{proof}
    \begin{equation}
\begin{multlined}
    \mathbb{E}_{T, \Delta |\mathbf{X} =\mathbf{x}}\left[\mathrm{L}_{k, \zeta}\left(\hat{F}_k(\zeta|\mathbf{x}), ( T, \Delta)\right)\right] = 
    \ell\left(\hat{F}_k(\zeta|\mathbf{x}), 1\right)
        \mathbb{P}(T^*\leq \zeta, \Delta = k | \mathbf{X}= \mathbf{x}) \\ 
     + \ell\left(\hat{F}_k(\zeta|\mathbf{x}), 0\right)
        \left( \mathbb{P}(T^*\leq \zeta, \Delta \neq k | \mathbf{X}= \mathbf{x})+ \mathbb{P}(T^*>\zeta| \mathbf{X} = \mathbf{x})\right)
\end{multlined}
\end{equation}

To be more explicit, we can define a new random variable $Y$: 
\begin{definition}
     $$\forall \zeta, ~ Y_{k, \zeta} \defeq T^* \leq \zeta \cap \Delta = k$$
\end{definition} 

\begin{equation}
    F^*_k(\zeta |\mathbf{x}) = \mathbb{P}(T^*\leq \zeta, \Delta = k| \mathbf{X} = \mathbf{x}) = \mathbb{P}(Y_{k, \zeta}=1| \mathbf{X} = \mathbf{x})
\end{equation}
$\hat{F}_k(\zeta|\mathbf{x})$ represents the estimated probability for $Y_{k, \zeta}=1$,
so we can rewrite: $\hat{p}_{k, \zeta} \defeq \hat{F}_k(\zeta|\mathbf{x}) \approx \mathbb{P}(Y_{k, \zeta}=1| \mathbf{X} = \mathbf{x})$
Therefore: 
\begin{flalign}
    \mathbb{E}_{T, \Delta |\mathbf{X} =\mathbf{x}}\left[\mathrm{L}_{k, \zeta}(\hat{F}_k(\zeta |\mathbf{x}), ( T, \Delta))\right] &=\mathbb{E}_{T, \Delta |\mathbf{X} =\mathbf{x}}[\mathrm{L}_{k,\zeta}(\hat{p}_{k, \zeta}, ( T, \Delta))] \\
    &=\ell\left(\hat{p}_{k, \zeta}, 0\right)
        \mathbb{P}(Y_{k, \zeta}=0| \mathbf{X} = \mathbf{x})+
    \ell\left(\hat{p}_{k, \zeta}, 1\right)
         \mathbb{P}(Y_{k, \zeta}=1| \mathbf{X} = \mathbf{x}) \\
    &= \mathbb{E}_{Y_{k, \zeta}}[\ell(\hat{p}_{k, \zeta}, Y_{k, \zeta})| \mathbf{X} = \mathbf{x}] \\
    &\leq \mathbb{E}_{Y_{k, \zeta}}[\ell(p_{k, \zeta}, Y_{k, \zeta})| \mathbf{X} = \mathbf{x}]\\
    &\leq \mathbb{E}_{T, \Delta |\mathbf{X} =\mathbf{x}}[\mathrm{L}_{k, \zeta}(\mathbb{P}(Y_{k, \zeta}=1), ( T, \Delta))] \\
    &\leq \mathbb{E}[\mathrm{L}_{k, \zeta}(F^*_k(\zeta |\mathrm{x}), ( T, \Delta))]
\end{flalign}

The last inequality is valid because $l$ is a proper scoring rule.  The same computation leads to a strictly proper scoring rule if $l$ is a strictly proper scoring rule. \\

So, we obtain that $\forall \zeta, \forall k \in \llbracket1, K\rrbracket, ~ \mathrm{L}_{k, \zeta}(\hat{F}_k(\zeta |\mathrm{x}), ( T, \Delta))$ is a proper scoring rule of $F_k^*(\zeta | \mathbf{x})$. \\
\end{proof}

\begin{theorem}\label{gpsr}
    If $\ell: \mathbb{R} \times \{0,1\} \rightarrow \mathbb{R}$, a chosen (strictly) proper scoring rule, thus $L_{\zeta}:  \mathbb{R} \times \mathcal{D} \rightarrow \mathbb{R}$ is a (strictly) proper scoring rule for the global CIF at a fixed time horizon $\zeta \in \mathbb{R}_+$.
\end{theorem}

\begin{proof}
    Straight forward thanks to the proposition and the lemma above.
\end{proof}

\paragraph{Corollary: Proper global scoring rule to compare competing risk models}
The defined scoring rule $\sum_{k=1}^K\mathrm{L}_{k, \zeta}$ is proper on
the time horizon $\zeta$ chosen arbitrarily. To be able to compare
different models, a global measure is necessary, \emph{eg} by summing
over time, as introduced in \citet{graf_assessment_1999}. Here, we extend the Integrated Brier Score to other (strictly) proper scoring rules $l$ and we prove that the Integrated Loss (IL) is also a (strictly) proper scoring rule. \\
By considering:
$$Z \sim \mathcal{U}(0, t_{max})$$ with $t_{max}$ the maximum time horizon for prediction.

\begin{definition}[\emph{Integrated global PSR}]
    With $\ell: \mathbb{R} \times \{0,1\} \rightarrow \mathbb{R}$, a chosen scoring rule, the cause-specific scoring rule function $L_{k, \zeta}:  \mathbb{R} \times \mathcal{D} \rightarrow \mathbb{R}$  defined as above, we define the $\mathrm{IL}$ as
    \begin{flalign}
        \mathrm{IL}(\hat{F}_1(.| \mathbf{x}), ..., \hat{F}_K(.| \mathbf{x}), ( T, \Delta))
        &\defeq \mathbb{E}_{Z}\left[\sum_{k=1}^K \mathrm{L}_{k, Z}(\hat{F}_{k}(Z| \mathbf{x}), ( T, \Delta))| \mathbf{X} = \mathbf{x}\right] \\
        &= \sum_{k=1}^K \underbrace{
        \mathbb{E}_{Z}\left[\mathrm{L}_{k, Z}(\hat{F}_k(Z| \mathbf{x}), ( T, \Delta))| \mathbf{X} = \mathbf{x}\right]
        }_{\defeq \mathrm{IL}_k(\hat{F}_{k}(.| \mathbf{x}), ( T, \Delta))}
    \end{flalign}
\end{definition}

\begin{corollary}
    With $\ell: \mathbb{R} \times \{0,1\} \rightarrow \mathbb{R}$, a chosen (strictly) proper scoring rule, the cause-specific loss function $L_{k, \zeta}:  \mathbb{R} \times \mathcal{D} \rightarrow \mathbb{R}$ defined above $\mathrm{IL}$ is a (strictly) proper scoring rule.
\end{corollary}

\begin{proof}
We have already proven that $L_{k, \zeta}:  \mathbb{R} \times \mathcal{D} \rightarrow \mathbb{R}$ is a (strictly) proper scoring rule. Using the monotonicity /positivity of the expectation, the result is immediate. 

\begin{align}
    \mathbb{E}_{T, \Delta | \mathbf{X}= \mathbf{x}, Z=\zeta} \left[\mathrm{IL}_k(\hat{F}_k(\zeta| \mathbf{x})), ( T, \Delta) \right] &= \mathbb{E}_{T, \Delta | \mathbf{X}= \mathbf{x}, Z=\zeta} \left[\mathrm{L}_k(\hat{F}_k(\zeta| \mathbf{x}), ( T, \Delta))\right] \\
    &\leq \mathbb{E}_{T, \Delta | \mathbf{X}= \mathbf{x}, Z=\zeta} \left[\mathrm{L}_k(F^*_k(\zeta| \mathbf{x}), ( T, \Delta))\right] \\
    &\leq \mathbb{E}_{T, \Delta | \mathbf{X}= \mathbf{x}, Z=\zeta} \left[\mathrm{IL}_k(F^*_k(\zeta| \mathbf{x}), ( T, \Delta))\right]
\end{align}
And because the expectation is non-decreasing, we have:
\begin{align}
    \mathbb{E}_{(\Delta, T)} \left[\mathrm{IL}_k(\hat{F}_k(Z| \mathbf{x}), (T, \Delta))| \mathbf{X}= \mathbf{x}\right] 
    &\leq \mathbb{E}_{(\Delta, T)} \left[\mathrm{IL}_k(F^*_k(Z| \mathbf{x}), (T, \Delta))| \mathbf{X}= \mathbf{x}\right]
\end{align}
This allows us to consider the IL as a global proper scoring rule to compare different competing risks models. 
\end{proof}

\section{The \citet{yanagisawa2023proper} scoring rule for survival}
\label{sec:cen_log_simple}

\citet{yanagisawa2023proper} introduce a metric, called $S_{Cen-log-simple}$, is an approximation of the proper scoring metric in \citet{rindt_survival_2022}.
Indeed, the metric in \citet{rindt_survival_2022} requires the hazard function, the time derivative of the cumulative incidence function, which is exposed only by differentiable models --and hence with an implicit assumption on almost-everywhere smooth time dependence. To avoid requesting this hazard function, \citet{yanagisawa2023proper} approximate it as piecewise affine. They show that under the assumption that the ``node time points'', edges of the affine, parts match an actual piecewise-affine breakdown of the CIF, the resulting approximation is proper. They argue that with enough node time points, the metric is a good approximation of a proper scoring rule.

$S_{Cen-log-simple}$ is defined as: 
\begin{multline}
    S_{Cen-log-simple}( \hat{F} , (t, \delta ); \{\zeta_i\}^B_{i=0}) \defeq \\
    -\delta \sum^{B-1}_{i=0} \mathbb{1}_{\zeta_i < t \leq \zeta_{i+1}} \log(\hat{F}(\zeta_{i+1}) - \hat{F}(\zeta_i)) \\
    - (1 - \delta) \sum_{i=0}^{B- 1} \mathbb{1}_{\zeta_i < t \leq \zeta_{i+1}} \log(1 - \hat{F}(\zeta_{i+1}))
\end{multline}
where $B$ is the number of node time points\footnote{We use $B=32$, as in the experiments in \citet{yanagisawa2023proper}}, and the $\{\zeta_i\}^B_{i=0}$ are the node times points, spaced between $0$ and $t_{max}$ to divide the space into $B$ equal intervals.

\section{Additional results for competing risk experiments}

\subsection{Results on synthetic dataset}%
\label{app:synthetic_results}

Varying the number of training points shows a slow improvement of SurvTrace, but at $n=5\cdot10^4$ MultiIncidence still has the best IBS (\autoref{ipsr}). 
MultiIncidence also maintains its benefit with an increased censoring rate (\autoref{censo}). In terms of computation time, MultiIncidence is the fastest, but the dependence on the number of features is similar across MultiIncidence, Fine \& Gray, and SurvTRACE (\autoref{ipsrvstime}).

\paragraph{Integrated Brier Score with a varying number of points}
By varying the number of training points in our synthetic dataset, while the Oracle Integrated Brier Score is decreasing, we see in Figure \ref{ipsr} that our method obtains better results than the transformer (SurvTRACE) in particular for a smaller number of training points. The number of training points may be a huge bottleneck for medical studies, as the number of patients may be low. We also see that, as the number of training points increases, SurvTRACE improves. With too many points, here 20,000, the Fine \& Gray model was too long to run. We also see that the Fine \& Gray model achieves approximately the same performance as our model, as expected because we model linear relations between the targets and the features.

\begin{figure}
\begin{minipage}{.25\linewidth}
\caption{\textbf{Integrated Brier Score (IBS) vs Training Samples on Synthetic Dataset} Integrated Brier Score for the synthetic dataset with linear relation over the features when we vary the number of samples. The test set was made into five different seeds.%
}%
\label{ipsr}
\end{minipage}%
\hfill%
\begin{minipage}{.74\linewidth}
    \includegraphics[width=\columnwidth]{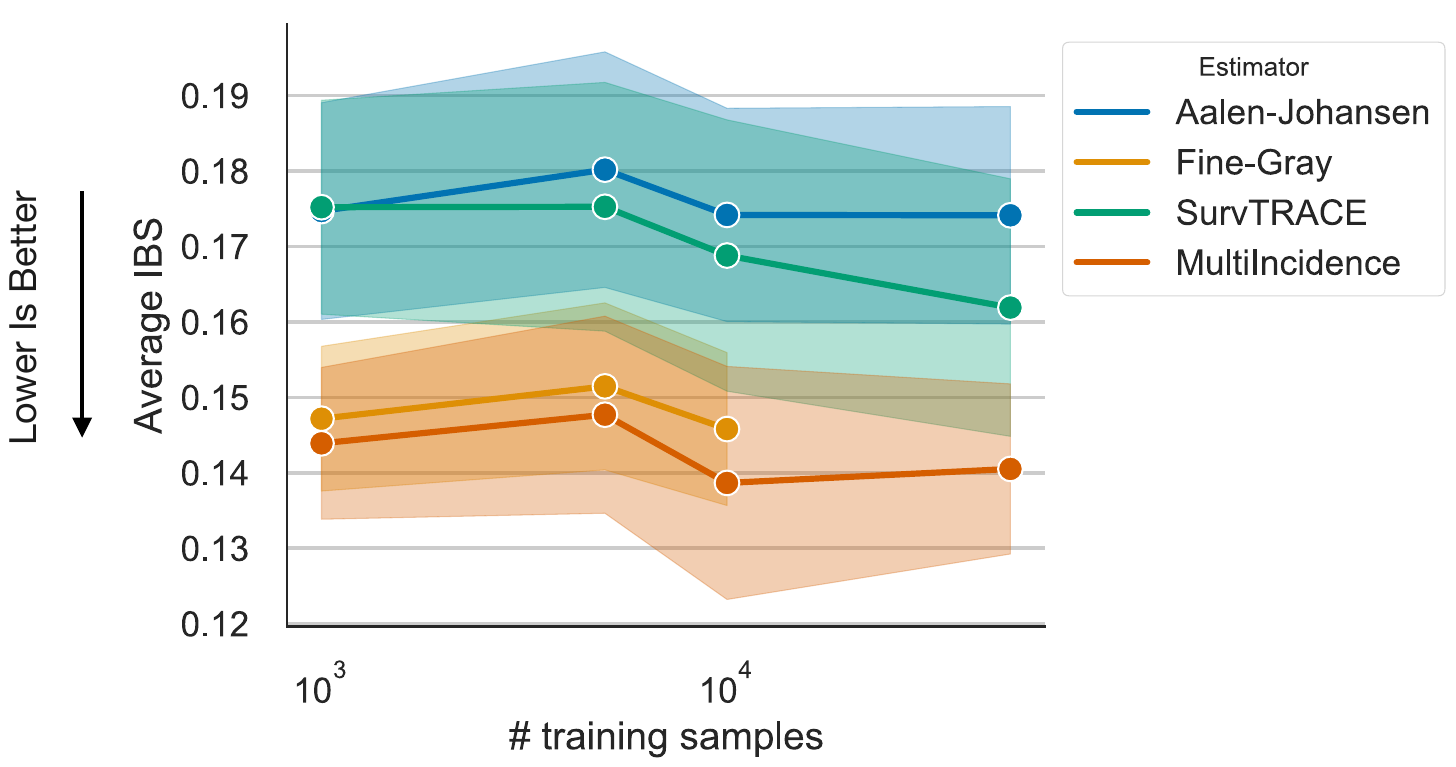}
\end{minipage}
\end{figure}

\paragraph{Computational cost vs performances}
To emphasize this phenomenon, we measured the time to fit each model, while varying the number of samples and the number of features in Figure \ref{fig:fit_predicttime_carre}. We show that for a limited number of samples, all of the methods take approximately the same amount of time to fit while having the worst results for SurvTRACE. With a higher number of samples, our method was faster to train than the other ones while achieving the same performance. We did not obtain the results for the Fine \& Gray model because the time to fit was higher than the given budget. 

\begin{figure}[b]
    \begin{minipage}{.25\linewidth}
	\caption{\textbf{Fitting time vs number of features} Time to fit 10,000 samples depending on the number of features.
	}
	\label{ipsrvstime}
    \end{minipage}%
    \hfill%
    \begin{minipage}{.74\linewidth}
	\includegraphics[width=\columnwidth]{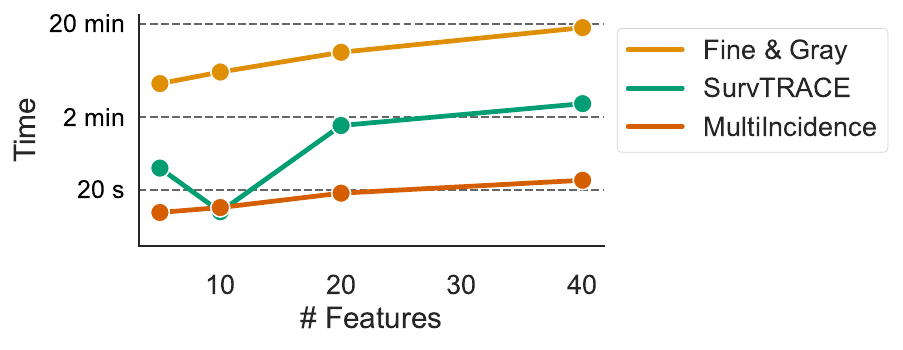}
    \end{minipage}%
\end{figure}

We show the dependence of time to fit with the number of features in Figure \ref{ipsrvstime}. In this figure, we highlight that our method takes less time to fit; the increase in time to fit with the number of features is similar among all methods. Another study of the impact of the features and the number of samples to fit the models can be found in Appendix \ref{fig:fit_predicttime}. 

\paragraph{Censoring Scale}
We studied the impact of censoring on the different models. To do so, we vary the censoring distribution to understand the effect of the learning scheme. In Figure \ref{censo}, we see that our method outperforms SurvTRACE at different censoring rates. As expected, all models get worse as the censoring rate increases.

\begin{figure}[t]
    \begin{minipage}{.25\linewidth}
	\caption{\textbf{Integrated Brier Score vs Censoring Rate} Integrated Brier Score for the synthetic dataset with 10,000 training points when we vary the censoring rate. Shaded areas represent the standard deviation across the different seeds. We used the Oracle censoring distribution to compute the weights
	\label{censo}}
    \end{minipage}%
    \hfill%
    \begin{minipage}{.74\linewidth}
	\includegraphics[width=\linewidth]{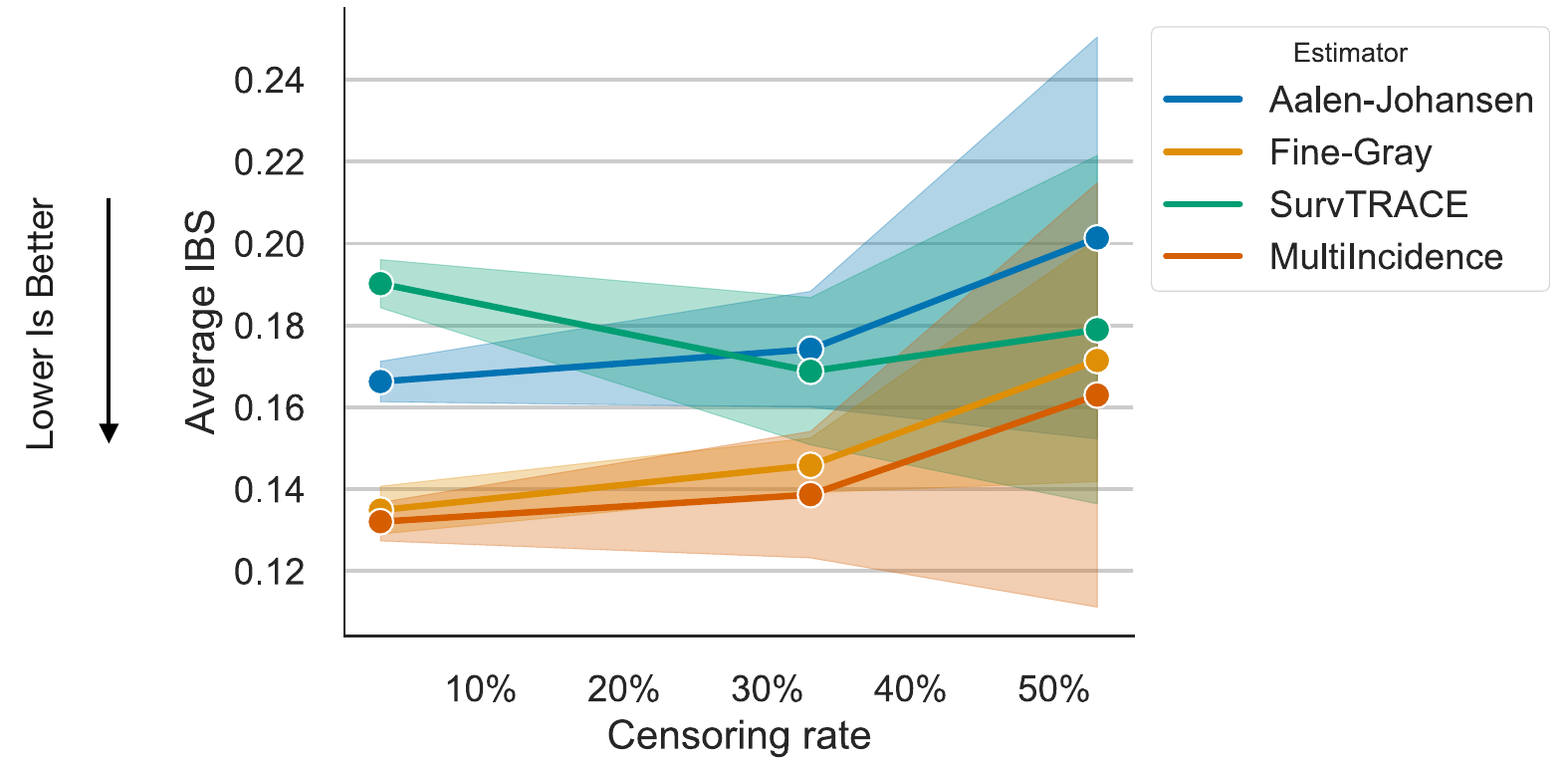}
    \end{minipage}%
\end{figure}

\paragraph{Brier Score in time}
We compared the Brier Score over time for each model, as shown in Figure \ref{bs}. The Brier Score increases over time for all models, which is expected due to the smaller number of individuals toward the end. Additionally, the associated weights contribute significantly to errors at later times. In this context, MultiIncidence consistently outperforms every other model for each event. 
\begin{figure}[t!]
\begin{minipage}{.25\linewidth}%
\caption{\textbf{Brier Score in time} Evolution of the Brier Score for the synthetic dataset \ref{synthedata} with 20,000 training points with 50\% of censoring. The weights are computed with the Oracle censoring distribution.}
\label{bs}
\end{minipage}%
\hfill%
\begin{minipage}{.74\linewidth}%
    \includegraphics[width=.9\linewidth]{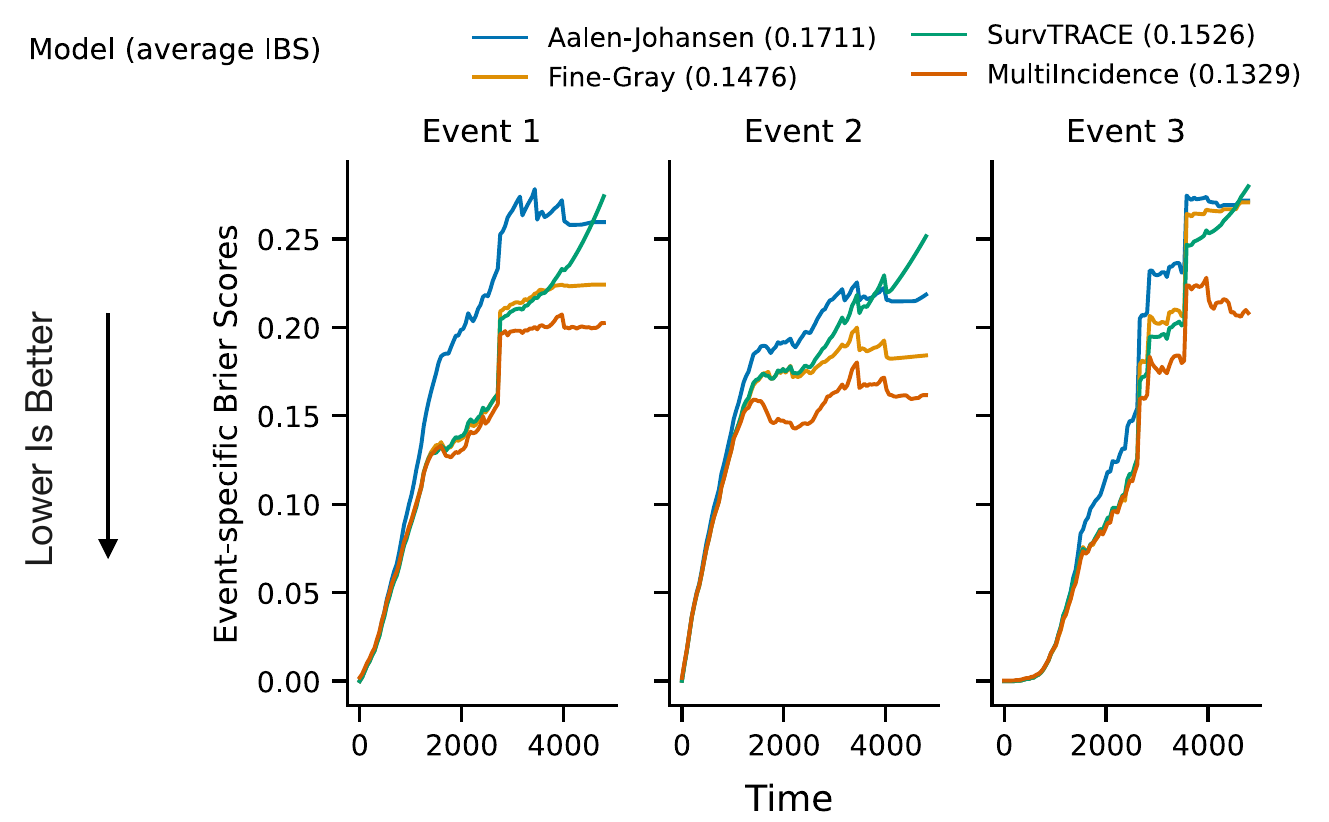}
\end{minipage}%
\end{figure}

\paragraph{Impact of the number of features and the training samples on fit time of competing risks}
\begin{figure}[ht]
    \centering
    \includegraphics[width=.8\columnwidth]{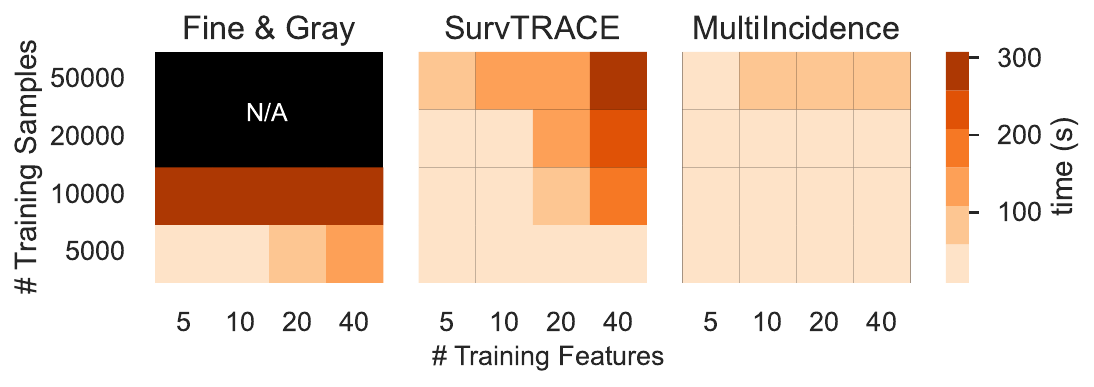}
    \caption{\textbf{Fit time for competing risks models}. We have measured the time to fit for each of them depending on the number of training points and the number of features.}
    \label{fig:fit_predicttime_carre}
\end{figure}

\begin{figure}[ht]\label{survtime}
\centerline{\includegraphics[width=\columnwidth]{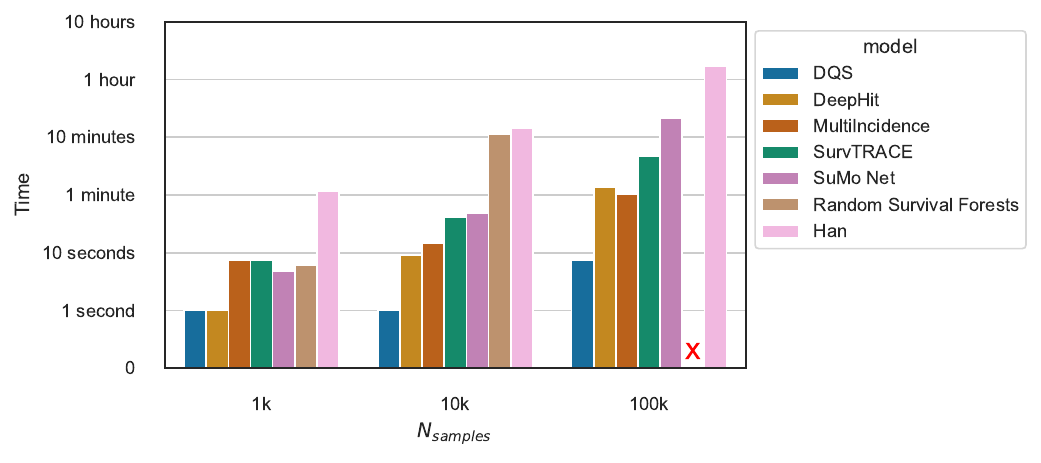}}
\caption{\textbf{Synthetic Dataset, training time for survival} Time to fit each survival method while varying the number of samples generated.}
\end{figure}

\subsection{Results for the SEER Dataset}%
\label{app:seer_results}

\paragraph{Learning curves}

We ran the experiments while varying the number of training points. In doing so, we measured the KM-adjusted Integrated Brier Score for each event. We also average it to have one global metric. We see in Figure  \ref{fig:ibs_seer} that our model of the global evaluation metric is quite stable and lower than the average Integrated Brier Score on SurvTRACE for any number of training points. We expanded the Integrated Brier Score for each event while training on the whole dataset except for the Random Survival Forests we trained with 100k data points and Fine and Gray with 10k data points because the last two methods could not handle such an amount of data. In Table \ref{tab:ibs_event_seer}, we compare our method with the other models. We see that our model MultiIncidence outperforms the other methods. Furthermore, figure \ref{fig:tradeoff_competing} shows that the models with the best average IBS are also the fastest to train.

\begin{figure}[b]
    \begin{minipage}{.3\linewidth}
    \caption{\textbf{Integrated Brier Score vs Number of Training Samples: SEER} Integrated Brier Score (Lower is Better) on the SEER dataset varying the number of samples: 50,000 samples, 100,000, and the full Training Dataset, aside for the Fine\&Gray model, which was tractable only for 10,000 samples.
    \label{fig:ibs_seer}}
    \end{minipage}%
    \hfill%
    \begin{minipage}{.69\linewidth}
	\includegraphics[width=\linewidth]{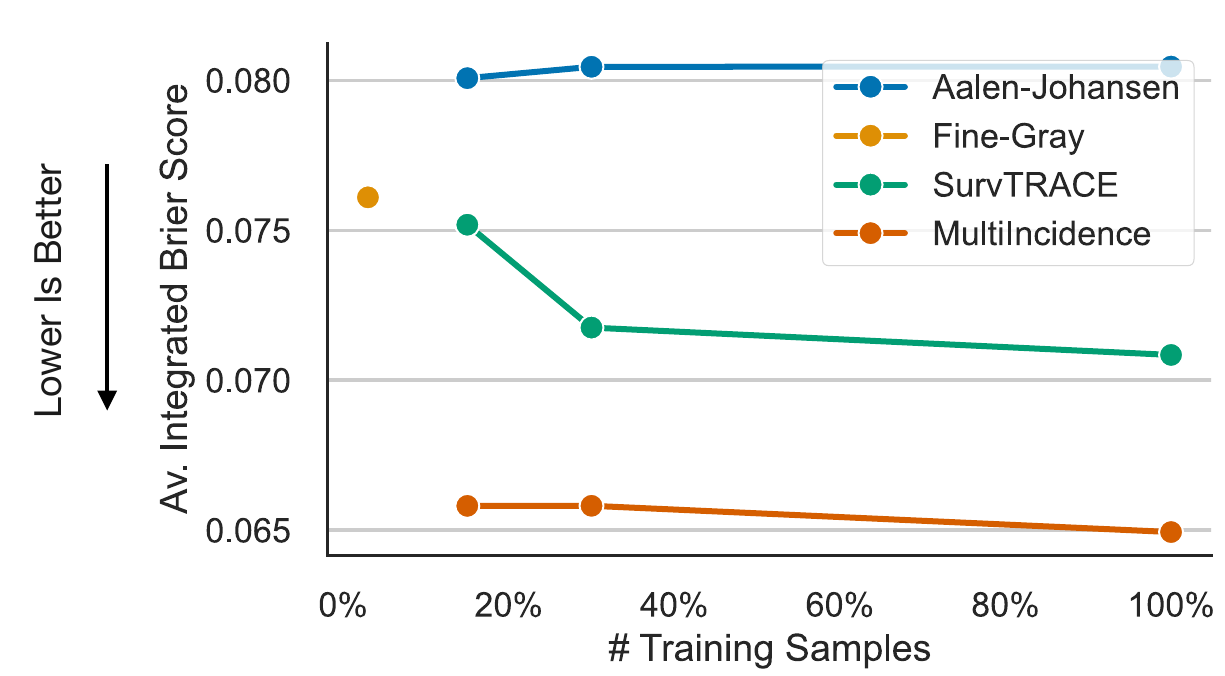}
    \end{minipage}%
\end{figure} 

\begin{table}
\begin{center}
\caption{Integrated Brier Score for each cause-specific risk on the SEER Dataset (Lower is Better). 
\label{tab:ibs_event_seer}}
\small\sc
\begin{tabular}{l|rrr}
\toprule
Event & 1 & 2 & 3 \\
\midrule
Aalen-Johansen & 0.1209 & 0.2832 & 0.0834\\
Fine \& Gray & 0.1055 & \underline{0.0281} & 0.0822 \\
Random Survival Forests& \textbf{0.0825}& 0.0295&0.0803 \\
DeepHit& 0.0931& 0.0330&0.0831 \\
DSM & 0.0875 & 0.0310 & 0.0869 \\
DeSurv & 0.0975 & 0.0327 & 0.0869 \\
SurvTRACE & 0.0871 & 0.0287 & \underline{0.0800} \\
MultiIncidence & \underline{0.0832} & \textbf{0.0273} & \textbf{0.0757} \\
\bottomrule
\end{tabular}
\end{center}
\end{table}

\paragraph{$C_{\zeta}$-index}

The $C$-index measures whether the ranking of the risk of the different
samples is in agreement with the order of the times in which the
event of interest happens\citep{harrell}. It
is originally a metric for survival settings but is often adapted to
competing risks settings where it is applied independently to each event \citep{uno_cstatistics_2011}.
In such settings, it is biased and does not control for the probabilities of
the events. However, as it is a popular metric, we have included it in
our experiments.

We give tables below for the $C_\zeta$-index toward time for the three
events \ref{tab:results_seer}. 
 At a fixed time horizon $\zeta$, we compute the $C_\zeta$-index for each class (corresponding to the ROC-AUC where we handled censored observations). The time horizons $\zeta$ are selected based on the any-event distribution, representing quantiles, indicating that at the time corresponding to 0.25, 25\% of events have already occurred. 
These results differ from those in the
SurvTRACE paper \citep{wang_survtrace_2022} for two reasons: \emph{1)} The
available code online only implements one of their losses, \emph{2)} they treated the SEER dataset with two competing risks, and any other event was classified as censored, instead of collapsing them in a third competing event.

\begin{table}[ht]
\begin{center}
\caption{C-index for competing risks on the SEER Dataset (Higher is Better)
\label{tab:results_seer}}
\small\sc
\begin{tabular}{l||ccc||ccc||ccc}
\toprule
Time-horizon quantile & & 0.25 & & & 0.50 & & & 0.75 & \\
\midrule
Event & 1 & 2 & 3 & 1 & 2 & 3 & 1 &2 &3 \\
\midrule
Aalen Johansen & 0.50 & 0.50 & 0.50 & 0.50 & 0.50 & 0.50 & 0.50 & 0.50 & 0.50 \\
Fine \& Gray & 0.80 & 0.67 &0.67 & 0.77 & 0.67 &0.69 &0.76 & 0.68 & 0.71 \\
Random Survival Forests & \textbf{0.89} & 0.79 & 0.79 & \textbf{0.87} & \textbf{0.78} & \textbf{0.77} & \textbf{0.85} & \textbf{0.77} & \textbf{0.77} \\
DeepHit & 0.83 & \textbf{0.86} & \textbf{0.85} & 0.75 &
0.75 & \underline{0.75} & 0.73 & 0.75 & \underline{0.75} \\
DSM & \underline{0.88} & \underline{0.85} & \underline{0.84} & 0.77 & 0.74 & \underline{0.75} & 0.76 & 0.75 & \underline{0.75} \\
DeSurv & 0.83 & 0.82 & 0.81 & 0.72 & 0.70 & 0.71 & 0.74 & 0.73 & 0.73 \\
SurvTRACE &  \underline{0.88} & 0.78& 0.77 &\underline{0.86} &\underline{0.76} &\underline{0.75} & \underline{0.84} &\underline{0.76} &\underline{0.75}\\
MultiIncidence & \underline{0.88} & 0.79 & 0.77 &
0.85 & 0.72 &0.71 & 0.81 & 0.66 &0.62\\
\bottomrule
\end{tabular}
\end{center}
\end{table}

\section{Additional results for survival experiments}

\subsection{Metrics for the survival analysis}\label{saall}
\begin{table}[ht]
\caption{\textbf{METABRIC}: Integrated Brier Score, $S_{Cen-log-simple}$ and c-index at 50\%}
\label{tab:metabric_supp}
\begin{center}
\begin{small}
\begin{sc}
\begin{tabular}{l|rrrrr}
\toprule
Model& C-index 0.25&C-index. 0.5&C-index 0.75&IBS & $S_{Cen-log-simple}$ \\
\midrule
Random Survival Forests&0.502±0.009&0.483±0.027&0.502±0.021&0.197±0.025&2.442±0.044\\
DeepHit&0.525±0.041&\textbf{0.639±0.024}&0.613±0.016&0.180±0.014&2.271±0.019\\
PCHazard&0.595±0.088&\textbf{0.639±0.019}&\textbf{0.639±0.014}&0.176±0.014&2.246±0.046\\
Han&0.626±0.035&0.622±0.007&0.628±0.006&0.191±0.003&2.420±0.150\\
DQS&0.601±0.019&0.630±0.032&0.633±0.014&0.180±0.034&\underline{2.205±0.044}\\
SuMo net&\textbf{0.660±0.022}&0.634±0.017&0.589±0.015&\underline{0.169±0.009}&2.302±0.059\\
SurvTRACE&0.589±0.082&0.627±0.015&0.629±0.007&\textbf{0.168±0.011}&2.270±0.034\\
MultiIncidence&\underline{0.627±0.016}&\underline{0.636±0.015}&\underline{0.635±0.011}&\textbf{0.168±0.019}&\textbf{2.169±0.056}\\ 
\bottomrule
\end{tabular}
\end{sc}
\end{small}
\end{center}
\end{table}

\begin{table}[ht]
\caption{\textbf{SUPPORT}: Integrated Brier Score and $S_{Cen-log-simple}$ (Lower is Better)}
\label{tab:support_supp}
\begin{center}
\begin{small}
\begin{sc}
\begin{tabular}{l|rrrrr}
\toprule
Model& C-index 0.25 & C-index 0.50& C-index 0.75 & IBS&$S_{Cen-log-simple}$ \\
\midrule
Random Survival Forests&0.481±0.024&0.527±0.019&0.531±0.020&0.225±0.004&1.942±0.023\\
DeepHit& 0.449±0.041 & \underline{0.609±0.004} & 0.599±0.003 & 0.217±0.005 & 2.251±0.021\\
PCHazard & 0.585±0.014&0.584±0.014&0.584±0.016&0.210±0.007&2.192±0.024\\
Han&0.576±0.016&0.574±0.007&0.587±0.011&0.260±0.012&3.483±0.307\\
DQS&\textbf{0.601±0.019}&0.598±0.012&0.592±0.009&0.201±0.007&1.987±0.069\\
SuMo net&\underline{0.590±0.016}&0.589±0.016&0.589±0.015&\underline{0.194±0.010}&\textbf{1.721±0.016}\\
SurvTRACE&0.578±0.008&\underline{0.609±0.005}&\underline{0.610±0.006}&\underline{0.194±0.005}&1.870±0.018\\
MultiIncidence&0.572±0.019&\textbf{0.618±0.007}&\textbf{0.615±0.007}&\textbf{0.191±0.006}&\underline{1.740±0.020}\\
\bottomrule
\end{tabular}
\end{sc}
\end{small}
\end{center}
\end{table}

\subsection{Trade-off between training time and performances}\label{tradeoffyana}
Here, we provide the results of our analysis of training time with the performances on the $S_{Cen-log-simple}$ of the different models for the survival analysis.

\begin{figure}[ht]
    \centering
    \includegraphics[width=0.8\columnwidth]{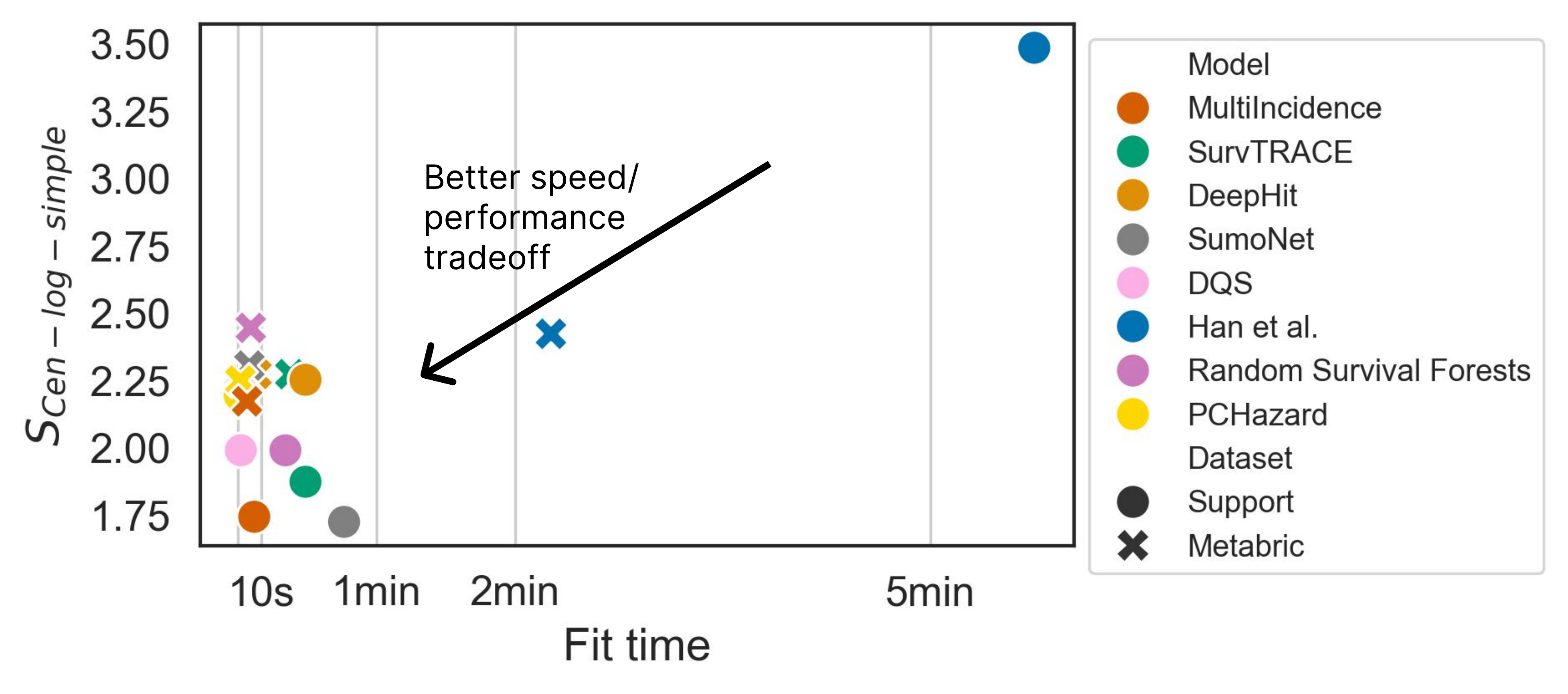}
    \caption{Trade-off between the performances and the training time for the $S_{Cen-log-simple}$ for the survival model over METABRIC and SUPPORT}
    \label{fig:fit_predicttime}
\end{figure}

\section{Implementation Details}

\subsection{Reference of used implementations for baselines}

We compare MultiIncidence with several baselines and describe their main characteristics and the implementation used in Table~\ref{tab:implem}

\begin{table}[ht]
\centering
\caption{Characteristics of used baselines.}
\begin{tabular}{|l|c|c|p{5cm}|p{3cm}|} \toprule
Name & \makecell{Competing\\ risks} & \makecell{Proper\\loss} & Implementation & Reference \\
\midrule
SurvTRACE & \checkmark & & ours & \citet{wang_survtrace_2022} \\ \midrule
DeepHit & \checkmark & & \rurl{github.com/havakv/pycox} & \citet{lee_deephit_2018}\\\midrule
DSM & \checkmark & & \rurl{autonlab.github.io/DeepSurvivalMachines} & \citet{nagpal2021deep}\\\midrule
DeSurv & \checkmark & & \rurl{github.com/djdanks/DeSurv}& \citet{danks_derivative-based_2022} \\\midrule
\makecell{Random Survival\\Forests} & \checkmark & & \rurl{scikit-survival.readthedocs.io/} for survival, and \rurl{www.randomforestsrc.org/} for competing risks & \citet{ishwaran_random_2008,ishwaran2014random} \\\midrule
Fine \& Gray & \checkmark & & \rurl{cran.r-project.org/package=cmprsk} & \citet{fine_proportional_1999}\\\midrule
Aalen-Johansen & \checkmark & & ours & \citet{aalen_survival_2008} \\\midrule
Han et al. & & & \rurl{github.com/rajesh-lab/Inverse-Weighted-Survival-Games} & \citet{han2021inverse}\\\midrule
PCHazard & & & \rurl{github.com/havakv/pycox} & \citet{kvamme_continuous_2019}\\\midrule
SumoNet & & \checkmark & \rurl{github.com/MrHuff/Sumo-Net} & \citet{rindt_survival_2022} \\\midrule
DQS & & \checkmark & \rurl{ibm.github.io/dqs/} & \citet{yanagisawa2023proper}\\
\bottomrule
\end{tabular}
\label{tab:implem}
\end{table}

\subsection{GridSearch Parameters}
We ran a Randomized Search for those parameters with a budget of 30. There are no parameters to tune for Aalen-Johansen and Fine \& Gray.
\begin{table}[ht]
    \centering
    \caption{Randomized Search Parameters}
    \begin{tabular}{|l|l||c|}
    \toprule
    Estimator & Parameter & Range \\
    \midrule
    MultiIncidence & Learning Rate & $loguniform(0.01, 0.5)$ \\
    & Nb of iterations & $\llbracket 20, 200 \rrbracket$ \\
    & Maximum Depth & $\llbracket 2, 10 \rrbracket$ \\
    & Nb of times & $\llbracket 1, 5 \rrbracket$ \\
    \midrule
    SurvTRACE & Learning Rate & $loguniform(10^{-5}, 10^{-3})$\\
    & Batch Size & $\{256, 512, 1024\}$\\
    & Hidden parameter & $\{2, 3\}$ \\ 
    \bottomrule
    \end{tabular}
    \label{tab:gridsearch}
\end{table}

\section{Distribution of the competing risks datasets}
\subsection{SEER Distribution of events}
Here, we present the distributions of the event of the SEER Dataset. We can highlight that the censoring distribution is non-uniform in time. The change in the censoring distribution from the $48^{th}$ month may be difficult to learn for some methods.
\begin{figure}[ht]
    \includegraphics[width=0.8\columnwidth]{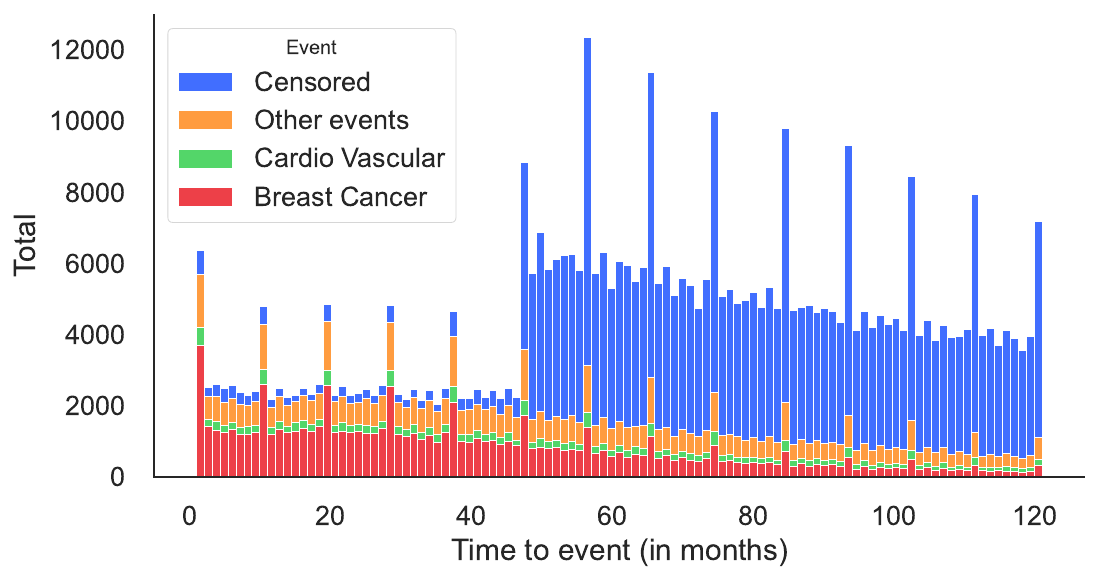}
    \caption{\textbf{SEER Dataset Distributions} The censoring rate is around 63\%. The prevalence of the different events is 18\% for Breast Cancer, 4.5\% for Cardio Vascular events, and 10\% for other events.}
    \label{fig:seer}
\end{figure}

\subsection{Example of distribution of one synthetic dataset}\label{synthedata}
\begin{figure}[ht]
\centerline{\includegraphics[width=0.8\columnwidth]{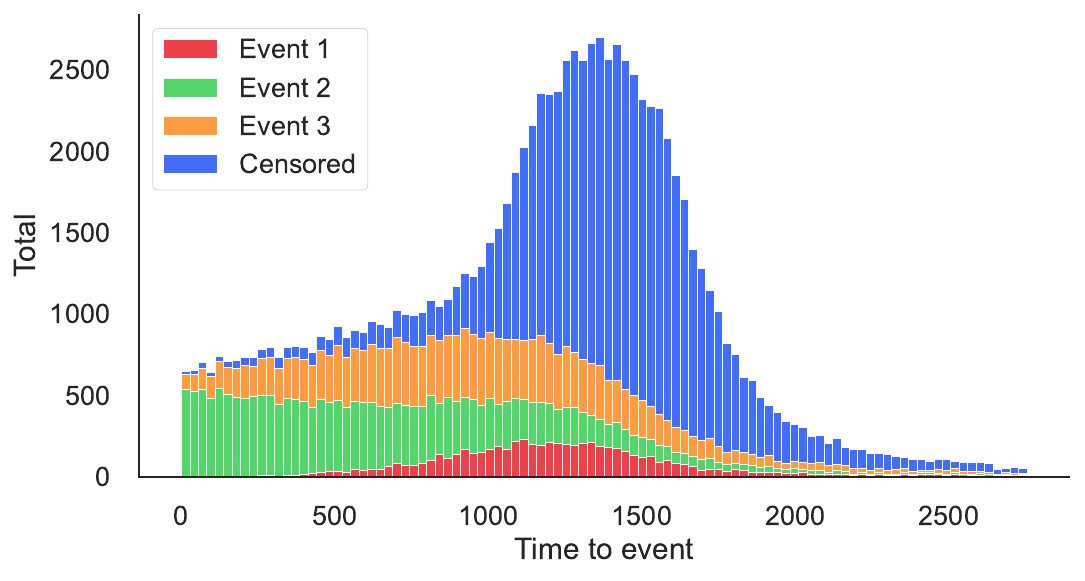}}
\caption{\textbf{Synthetic Dataset Distributions} Duration distributions of the synthetic dataset when censoring is dependent on X, censoring rate 69\%. The events are stacked.}
\label{distrib_synthe}
\end{figure} 

Figure \ref{distrib_synthe} shows an example of the distribution of the events with the censoring (dependent on the covariates). The parameters are chosen to fit three different behaviors possible. To illustrate this distribution, we can think of truck maintenance. Event 1, happening during the whole period duration, corresponds to the driver's driving skills. Event 2 may correspond to a misconception of the truck, happening from the beginning. Event 3 will refer to the truck's wear and tear. 

\newpage
\section{Corollary: Bregman divergence}
Here, we propose another proof with a scoring rule in the form of a Bregman Divergence. A Bregman divergence is a form of distance, and because of that, we want to minimize the Bregman divergence.

\begin{definition}
    Considering $U: \mathbb{R}^d \rightarrow \mathbb{R}$ strictly convex and differentiable, 
    \begin{flalign}
        \text{Bregman divergence} && 
        D_U(p, q) = U(p)-U(q)-\langle \nabla U(q), p-q\rangle. \geq 0
        && 
    \end{flalign}
\end{definition}
The specific choice of $l$ as $D_U$ does not change any computations of the expectation, so we obtain: 
\begin{flalign*}
    \mathbb{E}_{T, \Delta |\mathbf{X} =\mathbf{x}}\left(\mathrm{L}_{k, \zeta}\left(\hat{F}_k(\zeta|\mathbf{x}), ( T, \Delta)\right)\right) &= D_U\left(0, \hat{F}_k(\zeta|\mathbf{x})\right)
        (1 - F^*_k(\zeta |\mathbf{x}))+
    D_U\left(1, \hat{F}_k(\zeta|\mathbf{x})\right)
         F^*_k(\zeta |\mathbf{x}) \\ 
    &= (U(0)-U(\hat{F}_k(\zeta|\mathbf{x}))+\langle \nabla U(\hat{F}_k(\zeta|\mathbf{x})),\hat{F}_k(\zeta|\mathbf{x})\rangle)(1-F^*_k(\zeta |\mathbf{x})) \\ &\hspace{0.8cm}+  (U(1)-U(\hat{F}_k(\zeta|\mathbf{x}))-\langle \nabla U(\hat{F}_k(\zeta|\mathbf{x})), 1-\hat{F}_k(\zeta|\mathbf{x})\rangle)F^*_k(\zeta |\mathbf{x})  
    \\
    &= U(1)F^*_k(\zeta |\mathbf{x}) + U(0)(1-F^*_k(\zeta |\mathbf{x})) - U(\hat{F}_k(\zeta|\mathbf{x})) \\
    & \hspace{3cm}+ \langle \nabla U(\hat{F}_k(\zeta|\mathbf{x})), \hat{F}_k(\zeta|\mathbf{x}) - F^*_k(\zeta |\mathbf{x}) \rangle
\end{flalign*}
Meanwhile, because $U$ is strictly convex and differentiable: 
\begin{align}
    \forall p, \hat{p}, ~~~U(p) &> U(\hat{p}) + \langle \nabla U(\hat{p}), p - \hat{p} \rangle \\
    -U(\hat{p}) + \langle \nabla U(\hat{p}), \hat{p} - p \rangle &> -U(p) 
\end{align}

This implies: 
\begin{align*}
    \mathbb{E}_{T, \Delta |\mathbf{X} =\mathbf{x}}\left(\mathrm{L}_{k, \zeta}\left(\hat{F}_k(\zeta|\mathbf{x}), ( T, \Delta)\right)\right) &=D_U\left(0, \hat{F}_k(\zeta|\mathbf{x})\right)
        (1 - F^*_k(\zeta |\mathbf{x}))+
    D_U\left(1, \hat{F}_k(\zeta|\mathbf{x})\right)
         F^*_k(\zeta |\mathbf{x}) \\
    &> U(1)F^*_k(\zeta |\mathbf{x}) + U(0)(1-F^*_k(\zeta |\mathbf{x})) - U(F^*_k(\zeta |\mathbf{x})) \\
    &> D_U\left(0, F^*_k(\zeta|\mathbf{x})\right)
        (1 - F^*_k(\zeta |\mathbf{x}))+
    D_U\left(1, F^*_k(\zeta|\mathbf{x})\right)
         F^*_k(\zeta |\mathbf{x}) \\
    &> \mathbb{E}_{T, \Delta |\mathbf{X} =\mathbf{x}}\left(\mathrm{L}_{k, \zeta}\left(F^*_k(\zeta|\mathbf{x}), ( T, \Delta)\right)\right)
\end{align*}

We obtain that, a negative Bregman Divergence leads to a strictly proper scoring rule. 

\section{Examples}
\subsection{Brier Score}
When we define $l(y, \hat{y}) \defeq (y - \hat{y})^2$, we obtain the censoring adjusted Brier score for the $k^{th}$ competing event as define in Eq. 14 of \cite{kretowska_tree-based_2018}:

\begin{definition}
\begin{multline}
    \forall \zeta, \forall k \in [1, K], \\
    \mathrm{BS}_k(\hat{F}_k(\zeta, \mathbf{x}), \delta, t, \zeta, \mathbf{x}) \defeq \frac{1}{n} \sum_{i=1}^n 
    \dfrac{
        \mathbb{1}_{t_i \leq \zeta, \delta_i = k}\left(1 - \hat{F}_k(\zeta|\mathbf{x}_i)\right)^2
        }
        {G^*(t_i|\mathbf{x}_i)}
        +
        \dfrac{
        \mathbb{1}_{t_i > \zeta} \left(\hat{F}_k(\zeta|\mathbf{x}_i)\right)^2
        }
        {G^*(\zeta|\mathbf{x}_i)} 
     \\
    + 
        \dfrac{
        \mathbb{1}_{t_i \leq \zeta, \delta_i \neq 0, \delta_i \neq k} \left(  \hat{F}_k(\zeta|\mathbf{x}_i)\right)^2
        }
        {G^*(t_i|\mathbf{x}_i)}
\end{multline}
\end{definition}

\subsection{Binary cross entropy loss} 
As it is explained in \citet{benedetti_scoring_2010}, the log loss captures better the uncertainty than the mean squared error. So, one could also evaluate survival and competing risks models with the following loss.

\begin{multline}
    \forall k \in [1, K], \\
    \mathrm{l}_k(\hat{F}_k(\zeta, \mathbf{x}), \delta, t, \zeta)\defeq \frac{1}{n} \sum_{i=1}^n 
    \dfrac{
        \mathbb{1}_{t_i \leq \zeta, \delta_i = k}\log\left( \hat{F}_k(\zeta|\mathbf{x}_i)\right)
        }
        {G^*(t_i|\mathbf{x}_i)}
        + 
        \dfrac{
        \mathbb{1}_{t_i \leq \zeta, \delta_i \neq 0, \delta_i \neq k} \log\left(1 -  \hat{F}_k(\zeta|\mathbf{x}_i)\right)
        }
        {G^*(t_i|\mathbf{x}_i)} \\ 
        + 
        \dfrac{
        \mathbb{1}_{t_i > \zeta} \log\left(1 - \hat{F}_k(\zeta|\mathbf{x}_i)\right)
        }
        {G^*(\zeta|\mathbf{x}_i)}
\end{multline}

\end{document}